\documentclass[11pt,twoside]{article}
\usepackage[utf8]{inputenc}
\usepackage[right=1in, bottom=1.3in, top=1in]{geometry}
\usepackage{setspace}

\usepackage{amsfonts,amsmath,amssymb,amsthm}
\usepackage[shortlabels]{enumitem}
\usepackage[hidelinks]{hyperref}
\usepackage{graphics,graphicx,subfigure}
\usepackage{caption}
\usepackage{romannum}
\usepackage{color}
\usepackage{tikz}
\usepackage{multirow}
\usepackage{bbm}      
\usepackage{romannum}
\usepackage{algorithm, xcolor}
\usepackage{algpseudocode}

\theoremstyle{plain}
\theoremstyle{definition}
\newtheorem{lemma}{Lemma}[section]
\newtheorem{theorem}[lemma]{Theorem}

\newtheorem{definition}{Definition}

\newtheorem{problem}{Problem}

\newtheorem{assumption}{Assumption}
\newtheorem{remark}{Remark}

\theoremstyle{remark}

\newcommand{\R}{\mathbb{R}}
\newcommand{\C}{\mathbb{C}}
\newcommand{\E}{\mathbb{E}}

\newcommand{\Pbb}{\mathbb{P}}

\newcommand{\Ab}{\mathbf{A}}
\newcommand{\Bb}{\mathbf{B}}
\newcommand{\Ib}{\mathbf{I}}

\newcommand{\Yb}{\mathbf{Y}}
\newcommand{\Wb}{\mathbf{W}}

\newcommand{\xb}{\mathbf{x}}
\newcommand{\yb}{\mathbf{y}}
\newcommand{\cb}{\mathbf{c}}
\newcommand{\zb}{\mathbf{z}}
\newcommand{\ub}{\mathbf{u}}
\newcommand{\vb}{\mathbf{v}}

\newcommand{\Uc}{\mathcal{U}}
\newcommand{\Vc}{\mathcal{V}}
\newcommand{\Nc}{\mathcal{N}}
\newcommand{\Fc}{\mathcal{F}}

\newcommand{\Id}{\mathrm{\mathop{Id}}}
\newcommand{\argmin}{\mathop{\mathrm{argmin}}}
\newcommand{\vol}{\mathop{\mathrm{vol}}}
\newcommand{\indicator}{\mathbbm{1}}
\newcommand{\omegab}{\boldsymbol\omega}

\title{Cauchy Random Features for Operator Learning in Sobolev Space}
\author{Chunyang Liao, Deanna Needell, and Hayden Schaeffer}
\date{Department of Mathematics, \quad University of California, Los Angeles. \\ (liaochunyang@math.ucla.edu, deanna@math.ucla.edu, hayden@math.ucla.edu)}

\begin{document}
\pagenumbering{arabic}
\maketitle

\begin{abstract}
Operator learning refers to the applications of machine learning approaches to approximate operators between infinite dimensional Banach spaces.
While most progress in this area has been driven by variants of deep neural networks such as the Deep Operator Network (DeepONet) and Fourier Neural Operator (FNO), the theoretical guarantees are often in the form of a universal approximation property. 
However, the existence theorems do not guarantee that an accurate operator network is obtainable in practice. 
Motivated by the encoder-decoder framework, we propose a random feature operator learning framework.
Moreover, we provide a generalization error bounds for our proposed method depending on the number of random features.
Our experiments on several operator learning benchmarks demonstrate that our model can obtain similar or better test errors compared with state-of-the-art operator learning models and can significantly reduce the computational complexity.
\end{abstract}

\section{Introduction}

Approximating the solutions of partial differential equations (PDEs) with numerical algorithms is a fundamental problem in scientific and engineering applications. 
Traditional methods include finite difference, finite element, and spectral method which require access to the governing equation.
In recent years, the use of machine learning for solving PDEs has attracted attention, with several directions for addressing the approximation problems. 
Physics-informed neural network (PINN) \cite{Raissi2019PINN} trains a network to fit a single PDE task using a least squares fit on the data and on the PDE using a collocation method. However, any changes to the original problem leads to retraining the neural network, which is computationally expensive. 
Operator Learning trains a network to approximate the solution operator between the input and output function spaces directly. 
For example, we aim to approximate an operator between boundary condition and the PDE solution itself. The approximation of operators using neural network was first introduced in \cite{universal1995Chen}.
With the development of deep neural networks, recent results include various neural operators such as Deep Operator Nets \cite{DeepONet2021Lu}, Fourier Neural Operators \cite{Li2020FourierNO} and several others \cite{BelNet2023Zhang, Reduced2022Qian}. 
Neural operators have been used in several scientific applications, for example, solving spatio-temporal dynamics \cite{DeepONet2021Lu, wang2021DeepONet,Mao2023PPDONet,ZHU2023Neural}, dynamical system with control \cite{universal1995Chen}, reduced order modeling (ROM) \cite{LUCIA2004ROM}, climate predictions \cite{kurth2023climate} and uncertainty quantification \cite{YANG2022uncertainty}. 
While these neural networks based operator learning methods are often used in practice, the theoretical analysis relies on universal approximation property of deep neural networks, which shows the existence of a network of a requisite size achieving a certain error rate. However, the existence result does not guarantee that the network is obtained in practice, \cite{universal1995Chen, DENG2022approximation, kovachki2021approximation}. 
 Some prior results go beyond universal approximation by providing theoretical guarantees for training algorithms and sample complexity, see \cite{adcock2022adeep, Liu2024deep, Nicola2025practical,Liu2025generalization} and references therein.

To address the limited theoretical results in operator learning, a kernel or Gaussian Process (GP) based framework was proposed in \cite{BATLLE2024kernel, mora2024OLGP} along with a priori error estimates and convergence guarantees.  
In several benchmark tests, it was shown that the kernel-based approach matches or outperforms neural operator methods in terms of test accuracy and computational costs. 
Kernel methods operate on the kernel matrix (Gram matrix) of the data with size $m\times m$, where $m$ represents the number of samples in the dataset. This leads to poor scaling with the size of training dataset, i.e., large training sets incur significant computational and storage costs.

The class of random features is one of the most popular techniques to speed up kernel methods in large-scale problems. Rather than forming and computing solutions using the kernel matrix directly, the random feature method (RFM) \cite{Rahimi2007RFM, Rahimi2008RFM} maps data into a relatively low-dimensional randomized feature space, which significantly reduces the computation needed for training. 
Additionally, the random feature model can be viewed as a randomized two-layer neural network whose weights connecting input layer and single hidden layer are randomly generated from a known distribution rather than trainable parameters. 
Several approaches utilized RFMs for approximating solutions to PDEs \cite{chen2022RFM, nelson2021random, Nelson2024operatorRFM}. 
In \cite{chen2022RFM}, the authors utilized the random features as a randomized Galerkin method for solving the PDE. 
In \cite{nelson2021random} and a subsequent review paper \cite{Nelson2024operatorRFM}, the authors took the operator learning perspective to approximate the solutions of the PDEs. 
Their analysis relies on the theory of vector-valued reproducing kernels. As a consequence of \cite{lanthaler2023error}, a high probability non-asymptotic error bound was derived. 
However, this quantitative result is dependent on the assumption that the target operator belongs to a reproducing kernel Hilbert space corresponding to an operator-valued kernel, which may not hold in practice. 
Moreover, the random feature maps were carefully designed and adapted to different problems. However, training also requires samples in the frequency domain, which may not be available for the original problem.

In this paper, we propose a novel random features based operator learning method.
We use the random Fourier features \cite{Rahimi2007RFM} and randomly generate features from a known probability distribution to approximate the solution operator.
We provide the theoretical guarantees and experimental validation of our proposed method. 
Our contributions are summarized below.
\begin{itemize}
\item {\bf Random Feature Method for Operator Learning.}  We follow the existing operator learning framework in \cite{MHASKAR2023local, BATLLE2024kernel, Liu2025generalization}. Rather than using neural networks or kernel methods, we introduce the random features method for the operator learning problem addressing one of the open problems in \cite{BATLLE2024kernel}.
\item {\bf Error Analysis.} We derive an error bound for proposed approach in Section \ref{Sec:Operator}. Our error analysis relies on the generalization error of random feature model for function approximation (Section \ref{Sec:RF}), which has been of recent interest as well, see \cite{Li2018TowardsAU, richardson2024srmd, HASHEMI2023generalization, chen2024conditioning, liu2023random, saha2023harfe, xie2022shrimp, chen2022concentration, Rudi2017generalizationRFM, MEI2022RFM,Rahimi2008Uniform}. We consider the overparametrized setting (the number of random features $N$ is larger than the number of training samples $m$) and the min-norm interpolation problem (or the ridgeless limit as the regularization parameter $\lambda\to0^+$). Our analysis depends on the condition number of random feature matrix, which is similar as the theoretical results in \cite{chen2024conditioning}. To better match the expected smoothness of the solutions, we propose Cauchy random features which lead to mixed Sobolev spaces.
\item {\bf Improved Numerical Performance.} In Section \ref{Sec:experiment}, we compare our method with kernel based method and neural operator methods empirically. Numerical experiments on benchmark PDE problems show several advantages of our method: 1) our method is easy to implement, 2) expensive computational resources are not required, 3) the training involves solving several convex optimization problems (in parallel) which significantly reduces the training time, and 4) it is competitive in terms of test-accuracy. 
\end{itemize}
The paper is organized as follows. In Section \ref{Sec:framework}, we introduce the proposed random feature operator learning method. Then, we prove the generalization error bounds of the random feature method for function learning in Section \ref{Sec:RF}. Using the obtained error bounds in section \ref{Sec:RF}, we derive generalization error bounds for random feature operator learning in Section \ref{Sec:Operator}. Lastly, Section \ref{Sec:experiment} presents the numerical results and comparisons.

{\bf Notations.} We let $\R$ be the set of all real number and $\C$ be the set of all complex number where $i=\sqrt{-1}$ denotes the imaginary unit. We define the set $[N]$ to be all natural numbers smaller than of equal to $N$, i.e., $\{1,2,\dots,N\}$. 
We use bold letters to denote vectors or matrices, and denote the identity matrix of size $n\times n$ by $\Ib_n$. 
For any two vectors $\xb,\yb\in\C^d$, the inner product is denoted by $\langle \xb,\yb\rangle = \sum_{j=1}^d x_j\Bar{y}_j$, where $\xb = [x_1,\dots,x_d]^\top$ and $\yb = [y_1,\dots,y_d]^\top$. For a vector $\xb\in\C^d$, we denote by $\|\xb\|_p$ the $\ell_p$-norm of $\xb$ and for a matrix $\Ab\in\C^{m\times N}$ the (induced) $p$-norm is written as $\|\Ab\|_p$. 
The conjugate transpose of a matrix $\Ab\in\C^{m\times N}$ is denoted by $\Ab^*$. 
We use $\Uc$ to denote the input function set with domain $D_\Uc$, and $\Vc$ to denote the output function set with domain $D_\Vc$. The operator mapping functions in $\Uc$ to functions in $\Vc$ is denoted by $G$. The $L^2$ norm of a function over domain $D_\Uc$ is defined as $\|u\|_{L^2(D_\Uc)} = \sqrt{\int_{D_\Uc} |u(\xb)|^2 d\xb}$. We write $B_\delta(\Uc)$ for the ball of radius $\delta>0$ in a normed space $\Uc$ equipped with the $L^2$ norm.

\section{Problem Statement}
\label{Sec:framework}
The goal is to learn an unknown Lipschitz operator $G:\Uc\to\Vc$ between two function sets $\Uc$ and $\Vc$ from training samples $\{(u_j, v_j)\}_{j\in[M]} \subset \Uc\times\Vc$ such that $G(u_j)=v_j$ for all $j\in[M]$. We consider Lipschitz operators in the following sense.
\begin{assumption}
\label{Assmps:Lipschitz_G}
Let $\Uc$ and $\Vc$ be vector spaces equipped with $L^2$ norm.
Let $D_\Uc$ and $D_\Vc$ be the domain of functions in $\Uc$ and $\Vc$, respectively. 
We assume that $G:\Uc\to\Vc$ is a Lipschitz operator, i.e. there exists a constant $L_G>0$ such that
\begin{equation*}
\|G(u_1) - G(u_2)\|_{L^2(D_\Vc)} \leq L_G \|u_1 - u_2\|_{L^2(D_\Uc)}
\end{equation*}
for any $u_1,u_2\in\Uc$.
\end{assumption}

For our problem, we consider the setting where the input/output functions are only partially observed through a finite collection of function values at given collocation points. Denote collocation points in $D_\Uc$ and $D_\Vc$ by $\{\xb_j\}_{j\in[n]} \subset D_\Uc$ and by $\{\yb_j\}_{j\in[m]} \subset D_\Vc$, respectively. Then, we define sampling operators $S_\Uc:\Uc\to\R^n$ and $S_\Vc:\Vc\to\R^m$ as 
\begin{equation*}
S_\Uc(u) = [u(\xb_1),\dots,u(\xb_n)]^\top\in\R^n, \mbox{ and } S_\Vc(v) = [v(\yb_1),\dots,v(\yb_m)]^\top\in\R^m.
\end{equation*}
Indeed sampling operators are linear, and we further assume that $S_\Uc$ and $S_\Vc$ are bounded operators throughout this paper. We could also generalize to the situation where functions are accessible through arbitrary linear measurements. For the sake of simplicity, we will focus on point evaluations in this paper. We formalize the problem statement as follows:
\begin{problem}
We aim to learn unknown Lipschitz operator $G:\Uc\to\Vc$ from training data $\left\{( S_\Uc(u_j), S_\Vc(v_j))\right\}_{j\in[M]}$, where $S_\Uc:\Uc\to\R^n$ and $S_\Vc:\Vc\to\R^m$ are bounded linear sampling operators and functions $u_j\in\Uc, v_j\in\Vc$ satisfy $G(u_j)=v_j$ for all $j\in[M]$.    
\end{problem}

\subsection{Operator learning framework}
Our problem statement gives rise to a diagram for operator learning, which is depicted in Figure \ref{fig:diagram}. Here, the map $f:\R^n\to\R^m$ is defined explicitly as 
\begin{equation}
\label{operator_f}
f = S_\Vc\circ G\circ R_\Uc,
\end{equation}
where $R_\Uc:\R^n\to\Uc$ is a recover map. Notice that the mapping $f:\R^n\to\R^m$ is between finite-dimensional spaces, which is more amenable to numerical approximation. We make some assumptions on $f$ throughout this paper.
\begin{assumption}
\label{Assmps:f}
We assume that the map $f:\R^n\to\R^m$ takes the form $f(\ub) = [f_1(\ub),\dots, f_m(\ub)]^\top$ for any $\ub\in\R^n$, and hence each component $f_j:\R^n\to\R$ can be approximated separately. Moreover, we assume that each component $f_j:\R^n\to\R$ of $f$ is Lipschitz in the sense that there exists $L_j>0$ such that $|f_j(\ub_1) - f_j(\ub_2)| \leq L_j\|\ub_1 - \ub_2\|_2$ for any $\ub_1,\ub_2\in\R^n$.
\end{assumption}
After constructing recovery maps $R_\Vc:\R^m\to\Vc$, and $\hat{f}:\R^n\to\R^m$, we are able to define operator $\Hat{G}$ to approximate the target operator $G$ as 
\begin{equation*}
\Hat{G} = R_\Vc\circ \Hat{f} \circ S_\Uc.    
\end{equation*}
In \cite{BATLLE2024kernel}, the authors proposed the same framework to do operator learning where they used the optimal recovery maps in reproducing kernel Hilbert spaces (RKHS) as recovery maps $R_\Vc$, $R_\Uc$ and $\hat{f}$. 
Precisely, the map $\hat{f}:\R^n\to\R^m$, which is viewed as an approximation of $f$, is the optimal recovery map in a vector-valued RKHS. The optimal recovery map $R_\Vc:\R^m\to\Vc$ takes function values as inputs and returns a function in $\Vc$. 
As the theory of optimal recovery suggests, the optimal recovery map returns the kernel interpolation \cite{Foucart2022LearningFN}. 
Other function approximation/reconstruction techniques can be applied as well. For example, in \cite{bartolucci2023representation}, the authors used neural networks to approximate $f:\R^n\to\R^m$.

\begin{figure}[!htbp]
\centering
\begin{tikzpicture}
\tikzset{edge/.style = {->,> = latex}}
\tikzstyle{every node}=[font=\LARGE]
\node [font=\LARGE] at (8,15)(A) {$\Uc$};
\node [font=\LARGE] at (12,15)(B) {$\Vc$};
\node [font=\LARGE] at (8,12)(C) {$\R^n$};
\node [font=\LARGE] at (12,12)(D) {$\R^m$};
\draw [edge] (A) -- (B);
\draw [edge] (7.9,14.5) -- (7.9,12.5);
\draw [edge] (8.1,12.5) -- (8.1,14.5);
\draw [edge] (11.9,14.5) -- (11.9,12.5);
\draw [edge] (12.1,12.5) -- (12.1,14.5);
\draw [edge] (8.5,12.1) -- (11.5,12.1);
\draw [edge] (8.5,11.9) -- (11.5,11.9);
\node [font=\LARGE] at (10,15.5) {$G$};
\node [font=\LARGE] at (7.5,13.5) {$S_\Uc$};
\node [font=\LARGE] at (8.5,13.5) {$R_\Uc$};
\node [font=\LARGE] at (10,12.5) {$f$};
\node [font=\LARGE] at (10,11.5) {$\Hat{f}$};
\node [font=\LARGE] at (12.5,13.5) {$R_\Vc$};
\node [font=\LARGE] at (11.5,13.5) {$S_\Vc$};
\end{tikzpicture}
\caption{Commutative diagram of the operator learning setup. This diagram can be generalized to the mesh invariant setting where the measurement functionals are different at test time. We refer readers to \cite[Figure 3]{BATLLE2024kernel} for more details.}
\label{fig:diagram}
\end{figure}

\subsection{Random Feature Operator Learning}
The key idea of the random feature method is to draw $N$ random frequencies $\omegab_k\in\R^d$ from a density $\rho(\omegab)$, and then to construct an approximation as
\begin{equation}
\label{RF_model}
f(\cdot,\cb^\sharp) = \sum_{k=1}^N c_k^\sharp \exp \left(i\langle \omegab_k,\cdot\rangle\right).
\end{equation}
The random feature model can be viewed as a one layer neural network where the weights in the first layer are drawn from a known distribution and only the weights in the second layer are trainable parameters. Instead of using the exponential function as an activation function, we could also use rectified linear unit (ReLU) activation function \cite{Hsu2021reluRF}, or trigonometric functions \cite{HASHEMI2023generalization}. Originally, the random feature method was proposed to speed up large-scale kernel machines, see \cite{Rahimi2007RFM, Liu2022RandomFF} for details. 

We define the recovery map $R_\Uc:\ub\mapsto R_\Uc(\ub):=f(\cdot, \cb^\sharp(\ub))$, where
\begin{equation}
\label{opt_Ru}
\cb^\sharp(\ub) \in \argmin_{\cb\in\R^N} \|\cb\|_2 \qquad\mbox{ subject to } S_\Uc(f(\cdot,\cb)) = \ub.
\end{equation}
The coefficient vector $\cb^\sharp\in\R^N$ is indeed trained by solving the min-norm interpolation problem.
Recall that $S_\Uc:\Uc\to\R^n$ is the point evaluation operator, then the constraint reads as $R_\Uc(\ub)(\xb_i) = u_i$ for $i\in[n]$. Therefore, we can rewrite the min-norm interpolation problem and the optimal coefficient vector $\cb^\sharp\in\R^N$ is trained by solving
$$
\cb^\sharp(\ub) \in \argmin_{\Ab_x \cb=\ub} \|\cb\|_2,
$$ 
where matrix $\Ab_x$ is defined (component-wise) by $(\Ab_x)_{j,k} = \exp(i\langle \omegab_k, \xb_j\rangle)$. 
The solution is given by $\cb^\sharp = \Ab_x^*(\Ab_x\Ab_x^*)^{-1}\ub = \Ab_x^\dagger\ub$, where $\Ab_x^\dagger$ is the pseudo-inverse of matrix $\Ab_x$. 
We will prove that the matrix $\Ab_x\Ab_x^*$ is invertible with high probability in Section \ref{Sec:RF}, and hence pseudo-inverse $\Ab_x^\dagger$ is well-defined.
Following the same arguments, we define the recovery map $R_\Vc:\R^m\to\Vc$. The optimal coefficient vector $\cb^\sharp\in\R^N$ is trained by solving the following min-norm interpolation problem \footnote{To simplify the notation, we use the same number of random features $N$.} 
$$
\cb^\sharp \in \argmin_{\Ab_y \cb=\vb} \|\cb\|_2,
$$ 
where the matrix $\Ab_y$ is defined (component-wise) by $(\Ab_y)_{j,k} = \exp(i\langle \omegab_k,\yb_j\rangle)$. 
The optimal coefficient vector is computed by $\cb^\sharp = \Ab_y^*(\Ab_y\Ab_y^*)^{-1}\vb = \Ab_y^\dagger \vb$, where $\Ab_y^\dagger$ is the pseudo-inverse of matrix $\Ab_y$.

Suppose that Assumption \ref{Assmps:f} holds, we propose to approximate each $f_j$ separately using random feature model, and hence the vector-valued random feature map $\Hat{f}:\R^n\to\R^m$ is defined as 
\begin{equation}
\label{operator_rf}
\Hat{f}(\ub) = \left[ \Hat{f}_1(\ub), \dots, \Hat{f}_m(\ub)\right]^\top\in\R^m.
\end{equation}
Each $\Hat{f}_j:\R^n\to\R$ is a random feature approximation of $f_j:\R^n\to\R$ and takes the form
\begin{equation}
\label{hat_f}
\Hat{f}_j(\ub)= \sum_{k=1}^{N} c_k^{(j)} \exp(i\langle \omegab_k,\ub\rangle).
\end{equation}

To compute the coefficient vectors $\cb^{(j)}\in\R^N$ for all $j\in[m]$, we solve the following min-norm interpolation problem
\begin{equation}
\min_{\cb^{(1)},\dots,\cb^{(m)}\in\R^N} \sum_{j=1}^{m}\|\cb^{(j)}\|_2 \qquad\mbox{ s.t. } \Hat{f}(\ub_\ell) = \vb_\ell \quad \mbox{ for all } \ell\in[M]
\end{equation}
where $\ub_\ell\in\R^n$ and $\vb_\ell\in\R^m$ contain the function values at collocation points of the $\ell$-th data pair $(u_\ell,v_\ell)\in \Uc \times \Vc$, i.e 
\begin{equation*}
\ub_\ell = S_u(u_\ell) = [u_\ell(\xb_1), \dots, u_\ell(\xb_n)] \in\R^n, \mbox{ and } \vb_\ell = S_v(v_\ell) = [v_\ell(\yb_1), \dots, v_\ell(\yb_m)] \in\R^m.
\end{equation*}
Due to Assumption \ref{Assmps:f}, the training process can be done simultaneously in all components of $f$. 
Therefore, we consider a sequence of parallel min-norm interpolation problems
\begin{equation}
\label{parallel_c}
\min_{\cb^{(j)}\in\R^N} \|\cb^{(j)}\|_2 \quad\mbox{ s.t. } \Ab \cb^{(j)} = \vb^{(j)},
\end{equation}
where the matrix $\Ab\in\R^{M\times N}$ is defined (component-wise) by $\Ab_{\ell,k}=\exp(i\langle \omegab_k,\ub_\ell\rangle)$ and $\vb^{(j)} = [v_1(\yb_j),\dots,v_M(\yb_j)]\in\R^M$.
The solution to the min-norm interpolation problem is given by $\cb^{(j)} = \Ab^*(\Ab\Ab^*)^{-1}\vb^{(j)} = \Ab^\dagger \vb^{(j)}$ for each $j\in[m]$, where $\Ab^\dagger$ is the pseudo-inverse of $\Ab$.

\begin{algorithm}
\caption{Random Feature Operator Learning - Training}
\label{Alg:rf1}
\begin{algorithmic}
\State\textbf{Inputs:} Training data $\{ (S_\Uc(u_j), S_\Vc(v_j))\}_{j\in[M]}$, number of random features $N$, and probability distribution $\rho$
\State\textbf{Outputs:} random feature approximation $\hat{f}$
\State 1. Sample $N$ random features $\{ \omegab_k \}_{k\in[N]}$ from $\rho$ independently.
\State 2. Compute coefficient vectors $\cb^{(j)}$ by solving \eqref{parallel_c}, for example, using Cholesky decomposition. This step can be done in parallel.
\State 3. Use trained coefficient vectors $\cb^{(j)}$ and random features $\{ \omegab_k \}_{k\in[N]}$ to produce each component $\hat{f}$ taking form \eqref{hat_f}.
\State 4. Return $\hat{f}$ by stacking $\hat{f}_j$ using \eqref{operator_rf}.
\end{algorithmic}
\end{algorithm}

\begin{algorithm}
\caption{Random Feature Operator Learning - Inference}
\begin{algorithmic}
\State\textbf{Inputs:} Test function $u\in\Uc$, number of random features $N$, and distribution of features $\rho$
\State\textbf{Outputs:} random feature approximation of $G(u)$
\State 1. Get function values at collocation points using sampling operator $S_\Uc$, i.e., $\ub = S_\Uc(u)$
\State 2. Obtain random feature approximation $\hat{f}(\ub)$ by Algorithm \ref{Alg:rf1}
\State 3. Sample $N$ random features $\{ \omegab_k \}_{k\in[N]}$ from $\rho$ independently
\State 4. Train random feature approximation taking form \eqref{RF_model}. The coefficients vector $\cb^\sharp$ is trained by solving min-norm interpolation problem and the random feature model interpolates $\hat{f}(\ub)$ at collocation points.
\State 5. Return trained random feature model from Step 4.
\end{algorithmic}
\end{algorithm}

\section{Random Feature Generalization Error}
\label{Sec:RF}
In this section, we provide bounds on the generalization error of approximating a function using random feature model. 

\subsection{Set-up and Notations}
For a probability distribution $\rho$ associated with random feature $\omegab$ in $\R^d$, we define function space 
\begin{equation}
\label{target_space}
\Fc(\rho) := \left\{ f(\xb) = \int_{\R^d} \hat{f}(\omegab)\exp(i\langle \omegab, \xb\rangle)d\omegab : \|f\|^2_\rho = \E_\omega \left|\frac{\hat{f}(\omegab)}{\rho(\omegab)} \right|^2 <\infty \right\},
\end{equation}
where $\hat{f}(\omegab)$ is the Fourier transform of function $f$. Let $D\subset\R^d$ be a compact domain with finite volume, i.e., $\vol(D)<\infty$. 
Function $f:D\to\C$ has finite $\rho$-norm if it belongs to $\Fc(\rho)$. 
By Proposition 4.1 in \cite{Rahimi2008Uniform}, the function space $\Fc(\rho)$ is a reproducing kernel Hilbert space (RKHS) with associated kernel function
\begin{equation}
\label{kernel}
k(\xb,\yb) = \int_{\R^d} \exp(i\langle \omegab, \xb\rangle)\exp(-i\langle \omegab, \yb\rangle)d\rho(\omegab).
\end{equation}
We consider Cauchy random features which are drawn from the tensor-product Cauchy distribution with scaling parameter $\gamma>0$, whose density function is 
\begin{equation}
\rho(\omegab) := \prod_{j=1}^d \frac{1}{\pi\gamma(1+\omega_j^2/\gamma^2) }.
\end{equation}  
Here, $\omega_j$ is the $j$-th entry of $\omegab\in\R^d$. 
Using the tensor-product Cauchy distribution, the kernel function defined in \eqref{kernel} is indeed the Laplace kernel, i.e. $k(\xb,\yb) = \exp(-\gamma\|\xb-\yb\|_1)$. Hence, function space $\Fc(\rho)$ defined in \eqref{target_space} is the reproducing kernel Hilbert space corresponding to Laplace kernel.
The function space $\Fc(\rho)$ also relates to the well-known Sobolev space of mixed smoothness, which is an appropriate function space for studying PDEs \cite{brezis2011functional}.
Precisely, let $p>0$, we define the Sobolev spaces of mixed smoothness by \cite{potts2024anovaboostingrandomfourierfeatures},
$$H^p_{\text{mix}}(\mathbb{R}^d):= \left\{f:D\rightarrow \mathbb{C}^d \,\, \Big\rvert \,\, \|f\|_{H^p_{\text{mix}}(\mathbb{R}^d)}< \infty \right\}$$
where the associated norm is defined by:
\begin{equation*}
\|f\|^2_{H^p_{\text{mix}}(\mathbb{R}^d)}:=\int_{\mathbb{R}^d} \, |\hat{f}(\omegab)|^2 \prod_{i=1}^d \left(1+|\omega_i|^2\right)^p d\omegab.
\end{equation*}
If the scaling parameter of tensor-product Cauchy distribution $\gamma=1$, then we have $\|f\|_\rho = \|f\|_{H^1_{\text{mix}}}$. For arbitrary $\gamma>0$, we could derive that
\begin{equation*}
\|f\|_\rho^2 \leq \|f\|^2_{H^1_{\text{mix}}(\mathbb{R}^d)} \max\left(\frac{\pi}{\gamma}, \pi\gamma\right)^d.
\end{equation*}

We consider the regression problem which is to find an approximation of $f\in\Fc(\rho)$. 
Given a set of random weights $\{ \omegab_k\}_{k\in[N]}$ which are i.i.d. samples from tensor-product Cauchy distribution $\rho(\omegab)$, we train a random feature model
\begin{equation}
\label{rf_train}
f^\sharp(\xb) = \sum_{k=1}^N c_k^\sharp \exp(i\langle \omegab_k,\xb\rangle)
\end{equation}
using samples $(\xb_j, f(\xb_j))$ for $j\in[m]$. 
Let $\Ab\in\C^{m\times N}$ be the random feature matrix with $\Ab_{j,k}=\exp(i\langle \omegab_k, \xb_j\rangle)$ for $j\in[m]$ and $k\in[N]$. The coefficient vector $\cb^\sharp\in\R^N$ is trained by solving the min-norm interpolation problem:
\begin{equation}
\cb^\sharp \in \argmin_{\Ab\cb = \yb} \|\cb\|_2,
\end{equation}
where $\yb = [f(\xb_1), \dots,f(\xb_m)]\in\R^m$.
Our main goal of this section is to bound the generalization error $\|f-f^\sharp\|_{L^2(D)}$, which is defined as
\begin{equation*}
\|f-f^\sharp\|_{L^2(D)} = \sqrt{ \int_{D} |f(\xb)-f^\sharp(\xb)|^2 d\xb}.
\end{equation*}
We decompose the generalization error into two parts
\begin{equation}
\|f-f^\sharp\|_{L^2(D)} \leq \underbrace{ \| f - f^* \|_{L^2(D)} }_{ \text{Approximation Error} } + \underbrace{ \|f^* - f^\sharp\|_{L^2(D)} }_{ \text{Estimation Error} }
\end{equation}
by triangular inequality. The approximation $f^\sharp$ is the random feature model defined in \eqref{rf_train} and $f^*$ is the best random feature model that will be defined later. 
We first notice that we can write the target function $f(\xb)$ as
\begin{equation}
\label{target_f}
f(\xb) = \int_{\R^d} \hat{f}(\omegab)\exp(i\langle \omegab, \xb\rangle)d\omegab = \E_{\omegab\sim\rho} \left[ \frac{\hat{f}(\omegab)}{\rho(\omegab)}\exp(i\langle \omegab, \xb\rangle) \right].
\end{equation}
To simplify the notation, we let $\alpha(\omegab) = \frac{\hat{f}(\omegab)}{\rho(\omegab)}$ (defined on the $\omegab \in \text{supp}(\rho)$ and zero otherwise) and then define $\alpha_{\leq T}(\omegab) = \alpha(\omegab)\indicator_{\left|\alpha(\omegab)\right|\leq T}$. Therefore, we could write $\alpha_{>T} = \alpha(\omegab) - \alpha_{\leq T}(\omegab) $, i.e.
$$
\alpha_{\leq T}(\omegab) = \begin{cases}
    \alpha(\omegab) \quad & \mbox{ if } \left|\alpha(\omegab) \right| \leq T \\
    0 \quad & \mbox{ otherwise }
\end{cases} \quad \mbox{ and } \quad
\alpha_{> T}(\omegab) = \begin{cases}
    \alpha(\omegab) \quad & \mbox{ if } \left|\alpha(\omegab) \right| > T \\
    0 \quad & \mbox{ otherwise }.
    \end{cases}
$$
We define $f^*(\xb)$ as
\begin{equation}
\label{rf_best}
f^*(\xb) = 
\frac{1}{N}\sum_{k=1}^N \alpha_{\leq T}(\omegab_k) \exp(i\langle \omegab_k,\xb\rangle)
\end{equation}
where $\omegab_k$ are i.i.d samples from distribution $\rho(\omegab)$.
Note that the expectation is given by:
$$\E f^*(\xb) = \E_{\omegab} \left[ \alpha_{\leq T}(\omegab)\exp(i\langle \omegab,\xb\rangle) \right].$$

\subsection{Main Results}
Before showing the main results of this section, we first introduce fill-in distance of training samples and state assumptions on the samples and random features. Specifically, the samples are well-separated and the random features are randomly generated from the tensor-product Cauchy distribution. The assumptions are more amenable to solving PDE problems where the Sobolev space is an appropriate space and collocation points are usually non-random. 

\begin{definition}[{\bf Fill-in Distance}]
Let $X = \{\xb_j\}_{j\in[m]} \subset D$ be a collection of points in $D$ and we define fill-in distance of $X$ as
$$
h_X = \max_{\xb\in D} \min_{\xb'\in X} \|\xb-\xb'\|_2.
$$
\end{definition}

\begin{assumption}
\label{Assmps:Lipschitz_RF}
We assume that $f^*-f^\sharp$ is $L_f$-Lipschitz continuous in the sense that there exists $L_f>0$ such that 
$$\left| \left(f^*(\xb_1)-f^\sharp(\xb_1)\right) - \left(f^*(\xb_2)-f^\sharp(\xb_2)\right) \right| \leq L_f \|\xb_1 - \xb_2\|_2 $$ 
for any $\xb_1, \xb_2\in D$. 
\end{assumption}

\begin{assumption}
\label{Assmps:data}
The random feature matrix $\Ab$ is defined (component-wise) by $\Ab_{j,k} = \exp(i\langle \omegab_k, \xb_j\rangle)$ and the feature weights $\{ \omegab_k\}_{k\in[N]} \subset\R^d$ are sampled from the tensor-product Cauchy distribution with scaling parameter $\gamma>0$, and that for the data points $X = \{\xb_j\}_{j\in[m]} \subset D$ there is a constant $K>0$ such that $\|\xb_j-\xb_{j'}\|_2\geq K$ for all $j,j'\in[m]$ with $j\neq j'$.
\end{assumption}

We first state the main results of Section 3 where we provide a generalization error bound for random features model.

\begin{theorem}[{\bf Generalization Error Bound}]
\label{main_RF}
Consider a compact domain $D\subset\R^d$. Assume the data $X = \{\xb_j\}_{j\in[m]} \subset D$, the random features $\{ \omegab_k\}_{k\in[N]} \subset\R^d$, and the random feature matrix $\Ab\in\C^{m\times N}$ satisfy Assumption \ref{Assmps:data}. If the following conditions hold:
\begin{align}
& N \geq C\eta^{-2}m\log\left(\frac{2m}{\delta}\right) \\
& \gamma \geq \frac{1}{K}\log\left(\frac{m}{\eta}\right)
\end{align}
for some $\delta,\eta\in(0,1)$, where $C>0$ is a universal constant, then with probability at least $1-5\delta$, we have
\begin{equation}
\|f - f^\sharp\|^2_{L^2(D)} \leq 2\vol(D) I_1 + 4L_f^2h_X^2\vol(D) + 2Ch_X^d(1+2\eta)\left( \frac{2m}{1-2\eta}I_1 + 2I_2 \right)
\end{equation}
where
\begin{equation*}
\begin{aligned}
&I_1 := \frac{2\|f\|_\rho^2}{N} + \frac{32\|f\|_\rho^2\log^2(2/\delta)}{N} + \frac{4\|f\|^2_\rho\log(2/\delta)}{N} \leq \frac{96\|f\|_\rho^2\log^2(2/\delta)}{N}, \\
&I_2 := \|f\|_\rho^2 + 4\|f\|_\rho^2\log(2/\delta) + \sqrt{2\log(2/\delta)}\|f\|_\rho^2 \leq 12\log(2/\delta)\|f\|_\rho^2.
\end{aligned}
\end{equation*}
\end{theorem}
\begin{remark}
In summary, we have
$$
\|f - f^\sharp\|^2_{L^2(D)} \leq \mathcal{O}\left(\frac{1}{N}+h_X^2 + h_X^d\right)\|f\|^2_\rho.
$$
Compared with the generalization error bound in \cite{chen2024conditioning}, a key difference is that we do not consider random samples $\{\xb_j\}_{j\in[m]}$ since our problem is motivated from solving PDEs where the collocation points are fixed. This results in the appearance of fill-in distance $h_X$ in the generalization error bound.
\end{remark}

\begin{remark}
There is a connection between the fill-in distance $h_X$ and the cardinality of set $X$ in which the points are quasi-uniformly distributed; that is $h_X = m^{-1/d}$.     
\end{remark}

\begin{remark}
Consider the equidistant points $X = \{\xb_j\}_{j\in[m]} \subset D$, then the lower bound of pairwise distance, i.e., constant $K=m^{-1/d}$. 
\end{remark}

\subsection{Approximation Error Bound}
Theorem \ref{rf_term1} provides an upper bound of the approximation error, which is the error between target function $f$ and best random feature approximation $f^*$.
\begin{theorem}[{\bf Bound on Approximation Error}]
\label{rf_term1}
Let $f$ and $f^*$ be defined in \eqref{target_f} and \eqref{rf_best}, respectively. Then for all $T>0$ and $\xb\in D$, we have
\begin{equation}
|f(\xb) - f^*(\xb)|^2 \leq \frac{2\|f\|_\rho^4}{T^2} + \frac{32T^2\log^2(2/\delta)}{N^2} + \frac{4\|f\|^2_\rho\log(2/\delta)}{N} 
\end{equation}
with probability at least $1-\delta$. 
Furthermore, with probability at least $1-\delta$,
\begin{equation}
\label{first_term}
\| f - f^* \|_{L^2(D)}^2 \leq \left( \frac{2\|f\|_\rho^4}{T^2} + \frac{32T^2\log^2(2/\delta)}{N^2} + \frac{4\|f\|^2_\rho\log(2/\delta)}{N} \right) \vol(D). 
\end{equation}
\end{theorem}

The proof of Theorem \ref{rf_term1} is based on the following two lemmas.
   
\begin{lemma}
\label{term1_lem1}
Let $f$ and $f^*$ be defined as \eqref{target_f} and \eqref{rf_best}, respectively. Then for all $T>0$, we have
\begin{equation}
\left| f(\xb) - \E_{\omegab} f^*(\xb) \right|^2 \leq \frac{\left(\E_{\omegab}[\alpha(\omegab)^2]\right)^2}{T^2} = \frac{\|f\|_{\rho}^4}{T^2}.
\end{equation}
\end{lemma}
\begin{proof}
The equality is easy to check by the definitions of $\alpha(\omegab)$ and $\|f\|_\rho$.
The inequality follows from
\begin{equation}
\begin{aligned}
\left| f(\xb) - \E_{\omegab} f^*(\xb) \right|^2 =& \left| \E_{\omegab} \left[ \alpha_{>T}(\omegab)\exp(i\langle \omegab,\xb\rangle) \right] \right|^2 \\
\leq& \E_{\omegab} \left[\alpha(\omegab)\right]^2 \E_{\omegab} \left[ \indicator_{\left|\alpha(\omegab)\right|> T}\exp(i\langle \omegab,\xb\rangle)\right]^2  \\
=& \E_{\omegab} \left[\alpha(\omegab)\right]^2 \Pbb(\alpha(\omegab)^2>T^2) \\
\leq& \frac{\left(\E_{\omegab}[\alpha(\omegab)^2]\right)^2}{T^2}
\end{aligned}
\end{equation}
where we use the Cauchy-Schwarz inequality in the second line and Markov's inequality in the last line. 
\end{proof}

\begin{lemma}
\label{term1_lem2}
Let $f^*$ be defined as \eqref{rf_best}. For all $T>0$, the following inequality holds with probability at least $1-\delta$,
\begin{equation}
\left|f^*(\xb) - \E_{\omegab} f^*(\xb) \right|^2 \leq \frac{32T^2\log^2(2/\delta)}{N^2} + \frac{4\|f\|^2_\rho\log(2/\delta)}{N}.
\end{equation}
\end{lemma}
\begin{proof}
For each $\xb\in D$, we define random variable $Z(\omegab) = \alpha_{\leq T}(\omegab)\exp(i\langle \omegab,\xb\rangle)$ and let $Z_1, \dots, Z_N$ be $N$ i.i.d copies of $Z$ defined by $Z_k = Z(\omegab_k)$ for each $k\in[N]$.
By boundedness of $\alpha_{\leq T}(\omegab)$, we have an upper bound $|Z_k| \leq T$ for any $k\in[N]$. The variance of $Z$ is bounded above as
$$
\sigma^2 := \E_{\omegab} |Z-\E_{\omegab} Z|^2 \leq \E_{\omegab} |Z|^2 \leq \E_{\omegab} [\alpha(\omegab)^2] = \|f\|^2_\rho. 
$$
By Lemma A.2 and Theorem A.1 in \cite{lanthaler2023error}, it holds that, with probability at least $1-\delta$,
\begin{equation*}
\left|f^*(\xb) - \E_{\omegab} f^*(\xb) \right| = \left|\frac{1}{N}\sum_{k=1}^N Z_k - \E_{\omegab} Z \right| \leq \frac{4T\log(2/\delta)}{N} + \sqrt{\frac{2\|f\|^2_\rho\log(2/\delta)}{N}}.
\end{equation*}
Taking the square for both sides and using the inequality $(a+b)^2\leq 2a^2+2b^2$ give the desired result.
\end{proof}
\begin{proof}[Proof of Theorem \ref{rf_term1}]
For each $\xb\in D$, we decompose $|f(\xb) - f^*(\xb)|^2$ into two parts as
\begin{equation}
|f(\xb) - f^*(\xb)|^2 \leq 2 \underbrace{|f(\xb) - \E_\omega f^*(\xb)|^2}_{{\rm \Romannum{1}}} + 2 \underbrace{|\E_\omega f^*(\xb) - f^*(\xb)|^2}_{ {\rm \Romannum{2}} }.
\end{equation}
Bounds on term \Romannum{1} and \Romannum{2} are given by Lemma \ref{term1_lem1} and Lemma \ref{term1_lem2}, respectively. Adding the bounds together gives the desired result. 
Integrating the bound over domain $D$ leads to inequality \eqref{first_term}.
\end{proof}

\subsection{Estimation Error Bound}
In this section, we aim to bound the estimation error $\|f^*-f^\sharp\|_{L^2(D)}$. The error bound depends on the condition number of random feature matrix $\Ab$, which can be bounded by using the concentration properties of random feature matrix. We refer readers to \cite{chen2022concentration, chen2024conditioning} for more details about the concentration properties of random feature matrix. We first state the concentration results, and then derive an upper bound for the estimation error. 

\begin{theorem}[{\bf Concentration Property of Random Feature Matrix}]
\label{Thm_concentration}
Assume the data $X = \{\xb_j\}_{j\in[m]} \subset D$, the random features $\{ \omegab_k\}_{k\in [N]} \subset\R^d$, and the random feature matrix $\Ab\in\C^{m\times N}$ satisfy Assumption \ref{Assmps:data}. If the following conditions hold
\begin{align}
\label{cond_N}
& N \geq C\eta^{-2}m\log\left(\frac{2m}{\delta}\right) \\
\label{cond_gamma}
& \gamma \geq \frac{1}{K}\log\left(\frac{m}{\eta}\right)
\end{align}
for some $\delta,\eta\in(0,1)$, where $C>0$ is a universal constant. Then with probability at least $1-\delta$, we have
\begin{equation}
\label{concentration}
 \left\|\frac{1}{N}\Ab\Ab^* - \Ib_m\right\|_2 \leq 2\eta. 
\end{equation}
\end{theorem}
The direct consequence of this result is that each eigenvalue of $\frac{1}{N}\Ab\Ab^*$ is close to 1. Specifically, if the conditions in Theorem \ref{Thm_concentration} are satisfied, then 
\begin{equation*}
\left| \lambda_k\left(\frac{1}{N}\Ab\Ab^*\right) -1 \right| \leq 2\eta
\end{equation*}
with probability at least $1-\delta$ for all $k\in[m]$. Here, $\lambda_k(\frac{1}{N}\Ab\Ab^*)$ is the $k$-th eigenvalue of matrix $\frac{1}{N}\Ab\Ab^*$. Therefore, the matrix $\Ab\Ab^*$ is invertible with high probability and hence the pseudo-inverse $\Ab^\dagger$ is well-defined.

\begin{proof}[Proof of Theorem \ref{Thm_concentration}]
The main idea in the proof is to bound difference \eqref{concentration} by 
$$
\left\|\frac{1}{N}\Ab\Ab^* - \Ib_m\right\|_2 \leq \left\| \frac{1}{N}\Ab\Ab^* - \E_{\omegab} \left[ \frac{1}{N}\Ab\Ab^* \right] \right\|_2 + \left\| \E_{\omegab} \left[ \frac{1}{N}\Ab\Ab^* \right] - \Ib_m \right\|_2
$$
To bound the first term, we let $\Wb_\ell$ be the $\ell$-th column of $\Ab$. Define the random matrices $\{ \Yb_\ell\}_{\ell\in[N]}$ as 
\begin{equation}
\Yb_\ell = \Wb_\ell \Wb_\ell^* - \E_\omega[\Wb_\ell \Wb_\ell^*].
\end{equation}
Then $(\Yb_\ell)_{j,j} = 0$ and for $j,k\in[m]$,  
\begin{equation*}
\begin{aligned}
(\Yb_\ell)_{j,k} =& \exp(i\langle \omegab_\ell, \xb_j-\xb_k\rangle) - \E_{\omegab} [\exp(i\langle \omegab, \xb_j-\xb_k\rangle)] \\
=& \exp(i\langle \omegab_\ell, \xb_j-\xb_k\rangle) - \exp(-\gamma\|\xb_j-\xb_k\|_1),
\end{aligned}
\end{equation*}
where we use the characteristic function of the tensor-product Cauchy distribution in the second equality.
Note that $\Yb_\ell$ is self-adjoint and its induced $\ell_2$ norm is bounded by its largest eigenvalue.
By Gershgorin’s disk theorem, $\|\xb_j-\xb_k\|_2\geq K$ for $j,k\in[m]$, and condition \eqref{cond_gamma},
\begin{equation}
\begin{aligned}
\|\Yb_\ell\|_2 \leq& \max_{j\in[m]} \sum_{k\neq j} |\exp(i\langle \omegab_\ell, \xb_j-\xb_k\rangle) - \exp(-\gamma\|\xb_j-\xb_k\|_1)| \\
\leq& \max_{j\in[m]} \sum_{k\neq j} \left( 1+\exp(-\gamma\|\xb_j-\xb_k\|_1) \right) \\
\leq& \max_{j\in[m]} \sum_{k\neq j} \left( 1+\exp(-\gamma\|\xb_j-\xb_k\|_2) \right) \\
\leq& \max_{j\in[m]} m(1+\exp(-\gamma K)) \\
\leq& m+\eta
\end{aligned} 
\end{equation}
The variance term is bounded by
\begin{equation}
\begin{aligned}
\left\|\sum_{\ell=1}^N \E_{\omegab} [\Yb_\ell^2] \right\|_2 \leq& \sum_{\ell=1}^N \left\|\E_{\omegab} [\Yb_\ell^2]\right\|_2 = \sum_{\ell=1}^N \left\|\E_{\omegab} \left[\Wb_\ell \Wb_\ell^* \Wb_\ell \Wb_\ell^*\right] - \left( \E_{\omegab} \left[\Wb_\ell \Wb_\ell^*\right]\right)^2\right\|_2\\
=& \sum_{\ell=1}^N \left\| m\E_{\omegab} \Wb_\ell \Wb_\ell^* - \left( \E_{\omegab} \left[\Wb_\ell \Wb_\ell^*\right]\right)^2 \right\|_2 \\
\leq& N\left(m(1+\eta)+(1+\eta)^2\right).
\end{aligned}
\end{equation}
Here we use the fact that $\Wb_\ell$ is a vector with $\|\Wb_\ell\|_2 = \sqrt{m}$, which implies $\Wb_\ell \Wb_\ell^* \Wb_\ell \Wb_\ell^* = m \Wb_\ell \Wb_\ell^*$, and $\E_{\omegab} \Wb_\ell \Wb_\ell^*$ is self-adjoint whose $\ell_2$ norm is bounded by
\begin{equation}
\left\| \E_{\omegab} \Wb_\ell \Wb_\ell^* \right\|_2 \leq 1 + \max_{j\in[m]} \sum_{k\neq j} |\exp(-\gamma \|\xb_k-\xb_j\|_1)| \leq 1 + m\exp(-\gamma \|\xb_k-\xb_j\|_2 ) \leq 1+\eta
\end{equation}
by Gershgorin’s disk theorem. Since $\{\Yb_\ell\}_{\ell\in[N]}$ are independent mean-zero self-adjoint matrices, applying matrix Bernstein's inequality gives 
\begin{equation}
\begin{aligned}
\Pbb \left( \left\| \frac{1}{N}\Ab\Ab^* - \E_{\omegab} \left[\frac{1}{N}\Ab\Ab^*\right] \right\|_2 \geq \eta  \right) =& \Pbb \left( \left\|\sum_{\ell=1}^N \Yb_\ell \right\|_2 \geq N\eta \right)\\ 
\leq & 2m\exp\left( - \frac{N\eta^2/2}{m(1+\eta)+(1+\eta)^2 + (m+\eta)\eta/3} \right) \\
\leq& 2m\exp\left( - \frac{N\eta^2}{5m+9} \right).
\end{aligned}
\end{equation}
The left-hand term is less than $\delta$, provided condition \eqref{cond_N} is satisfied with $C=6$ by assuming that $m\geq 9$ and $\eta<1$.

For the second term, denote by $\Bb$ the matrix $\E_{\omegab} [\frac{1}{N}\Ab\Ab^*] - \Ib_m$. Then $\Bb$ is symmetric and $\Bb_{j,j}=0$ and $\Bb_{j,k} = \E_{\omegab} \exp(i\langle \omegab, \xb_j-\xb_k\rangle) = \exp(-\gamma\|\xb_j-\xb_k\|_1)$ for all $j,k\in[m]$ with $j\neq k$. 
Note that $\Bb$ is self-adjoint, by Gershgorin's disk theorem and condition \eqref{cond_gamma}, the induced $\ell_2$ norm is bounded by 
\begin{equation*}
\|\Bb\|_2 = \max_{j\in[m]}\sum_{k\neq j} |\exp(-\gamma\|\xb_j-\xb_k\|_1)| \leq \max_{j\in[m]}\sum_{k\neq j}|\exp(-\gamma\|\xb_j-\xb_k\|_2)| \leq m\exp(-\gamma K) \leq \eta. 
\end{equation*}
\end{proof}
Utilizing the concentration property of random feature matrix $\Ab\in\C^{m\times N}$, we derive an upper bound for the estimation error $\|f^* - f^\sharp\|_{L_2(D)}$, which is summarized in Theorem \ref{rf_term2}. 
\begin{theorem}[{\bf Bound on Estimation Error}]
\label{rf_term2}
Assume the data $X = \{\xb_j\}_{j\in[m]} \subset D$, the random features $\{ \omegab_k\}_{k\in[N]} \subset\R^d$, and the random feature matrix $\Ab\in\C^{m\times N}$ satisfy Assumption \ref{Assmps:data}.
Let $f^*$ and $f^\sharp$ be defined as \eqref{rf_best} and \eqref{rf_train}, respectively, and satisfy Assumption \ref{Assmps:Lipschitz_RF}. If for some $\eta,\delta>0$ the following conditions hold, 
\begin{align}
& N \geq C\eta^{-2}m\log(\frac{m}{2\delta}) \\
& \gamma \geq \frac{1}{K}\log(\frac{m}{\eta})
\end{align}
where $C>0$ is a universal constant independent of the dimension $d$, then the following bound holds with probability at least $1-4\delta$
\begin{equation}
\|f^* - f^\sharp\|^2_{L^2(D)} \leq 2L_f^2h_X^2\vol(D) + Ch_X^d(1+2\eta)\left( \frac{2m}{1-2\eta}I_1 + 2I_2 \right),
\end{equation}
where quantities $I_1$ and $I_2$ are 
\begin{equation*}
\begin{aligned}
&I_1 := \frac{2\|f\|_\rho^4}{T^2} + \frac{32T^2\log^2(2/\delta)}{N^2} + \frac{4\|f\|^2_\rho\log(2/\delta)}{N}, \\
&I_2 := \|f\|_\rho^2 + \frac{4T^2\log(2/\delta)}{N} + \sqrt{\frac{2T^2\|f\|_\rho^2\log(2/\delta)}{N}}.
\end{aligned}
\end{equation*}
\end{theorem}

\begin{lemma}
\label{lem_second_term}
Assume the data $X = \{\xb_j\}_{j\in[m]} \subset D$, the random features $\{ \omegab_k\}_{k\in[N]} \subset\R^d$, and the random feature matrix $\Ab\in\C^{m\times N}$ satisfy Assumption \ref{Assmps:data}.
Let $f^*$ and $f^\sharp$ be defined as \eqref{rf_best} and \eqref{rf_train}, respectively, and satisfy Assumption \ref{Assmps:Lipschitz_RF}. Then we have
\begin{equation}
\| f^* - f^\sharp\|_{L^2(D)}^2 \leq 2L_f^2h_X^2\vol(D) + Ch_X^d\|\Ab\|_2^2 \| \cb^* - \cb^\sharp\|_2^2.
\end{equation}
\end{lemma}
\begin{proof} 
Consider a partition\footnote{If $\xb\in D$ belongs to more than one $D_j$, we randomly assign it to one of them to make $\{D_j\}_{j\in[m]}$ a partition. } $\{ D_j \}_{j\in[m]}$ of domain $D$ where 
\begin{equation*}
D_j = \{\xb\in D: \|\xb-\xb_j\| \leq \|\xb-\xb_k\| \mbox{ for all } k\neq j \}.
\end{equation*}
For each $\xb\in D$, we use Assumption \ref{Assmps:Lipschitz_RF} and apply triangle inequality to obtain
\begin{equation*}
\begin{aligned}
|f^*(\xb) - f^\sharp(\xb)| &\leq \left| \left(f^*(\xb) - f^\sharp(\xb)\right) - \left(f^*(\xb_j) - f^\sharp(\xb_j)\right)\right| + \left|f^*(\xb_j) - f^\sharp(\xb_j)\right| \\
& \leq L_f\|\xb-\xb_j\|_2 + \left|f^*(\xb_j) - f^\sharp(\xb_j)\right|
\end{aligned}
\end{equation*}
Then, the approximation error is bounded as
\begin{equation}
\label{riemann_error}
\begin{aligned}
\|f^* - f^\sharp\|^2_{L^2(D)} &= \int_{D} \left|f^*(\xb)-f^\sharp(\xb)\right|^2 d\xb = \sum_{j=1}^m \int_{D_j} \left|f^*(\xb)-f^\sharp(\xb)\right|^2 d\xb \\
& \leq \sum_{j=1}^m \int_{D_j} 2L_f^2\|\xb-\xb_j\|_2^2 + 2\left|f^*(\xb_j) - f^\sharp(\xb_j)\right|^2 d\xb \\
& \leq 2L_f^2h_X^2\vol(D) + \sum_{j=1}^m 2\vol(D_j)\left|f^*(\xb_j) - f^\sharp(\xb_j)\right|^2 \\
& \leq 2L_f^2h_X^2\vol(D) + Ch_X^d\sum_{j=1}^m \left|f^*(\xb_j) - f^\sharp(\xb_j)\right|^2 \\
& \leq 2L_f^2h_X^2\vol(D) + Ch_X^d\|\Ab\cb^* - \Ab\cb^\sharp\|_2^2 \\
& \leq 2L_f^2h_X^2\vol(D) + Ch_X^d\|\Ab\|^2 \| \cb^* - \cb^\sharp\|_2^2
\end{aligned}
\end{equation}
The third line holds since each $\xb\in D_j$ satisfies 
\begin{equation*}
\|\xb - \xb_j \|_2 = \min_{\xb'\in X} \|\xb-\xb'\|_2 \leq h_X,
\end{equation*}
and the fourth line holds since the volume of each $D_j$ is bounded by $Ch_X^d$ for some constant $C>0$. Although the constant $C$ depends on the dimension $d$, for the problems consider here, $d\leq3$. Specifically, the constant $C=2$ when $d=1$, $C=\pi$ when $d=2$, and $C=\frac{4\pi}{3}$ if $d=3$. 
\end{proof}

\begin{lemma}[{\bf Decay Rate of $\|\cb^*\|_2^2$}]
\label{decay_c}
Let $f^*$ be defined as \eqref{rf_best} and $\cb^*$ be the corresponding coefficient vector. For some $\delta\in(0,1)$, it holds with probability at least $1-\delta$ that
\begin{equation*}
\|\cb^*\|_2^2 \leq \frac{1}{N} \left( \|f\|_\rho^2 + \frac{4T^2\log(2/\delta)}{N} + \sqrt{\frac{2T^2\|f\|_\rho^2\log(2/\delta)}{N}} \right).
\end{equation*}
\end{lemma}
\begin{proof}
By the definition of $c_k^*$, we have
\begin{equation}
\|\cb^*\|_2^2 = \sum_{k=1}^N |c_k^*|^2 = \frac{1}{N^2} \sum_{k=1}^N |\alpha_{\leq T}(\omegab_k)|^2.     
\end{equation}
Define random variable $Z(\omegab) = |\alpha_{\leq T}(\omega)|^2$ and let $Z_1,\dots, Z_N$ be $N$ i.i.d copies of $Z$ defined as $Z_k = |\alpha_{\leq T}(\omegab_k)|^2$ for each $k\in[N]$.
By the boundedness of $\alpha_{\leq T}(\omegab)$, we have an upper bound $|Z_k|\leq T^2$ for each $k\in[N]$. The variance of $Z$ is bounded above as
\begin{equation*}
\sigma^2 := \E_{\omegab} |Z-\E_{\omegab} Z|^2 \leq \E_{\omegab} |Z|^2 \leq T^2\E_{\omegab} \left[ |\alpha(\omegab)|^2 \right] = T^2\|f\|^2_\rho.
\end{equation*}
By Lemma A.2 and Theorem A.1 in \cite{lanthaler2023error}, it holds with probability at least $1-\delta$ that
\begin{equation*}
\left| \frac{1}{N} \sum_{k=1}^N Z_k - \E_{\omegab} Z \right| \leq \frac{4T^2\log(2/\delta)}{N} + \sqrt{\frac{2T^2\|f\|_\rho^2\log(2/\delta)}{N}}.
\end{equation*}
Then, we have
\begin{equation*}
\|\cb^*\|_2^2 = \frac{1}{N} \left( \frac{1}{N}\sum_{k=1}^N |\alpha_{\leq T}(\omegab_k)|^2 \right) \leq \frac{1}{N} \left( \|f\|_\rho^2 + \frac{4T^2\log(2/\delta)}{N} + \sqrt{\frac{2T^2\|f\|_\rho^2\log(2/\delta)}{N}} \right).
\end{equation*}
\end{proof}

Now, we are ready to prove Theorem \ref{rf_term2} which states an upper bound for the estimation error. Recall the result of Lemma \ref{lem_second_term}, it remains to bound $\|\Ab\|_2^2$ and $\|\cb^*-\cb^\sharp\|_2^2$.

\begin{proof}[Proof of Theorem \ref{rf_term2}.]
Using Theorem \ref{Thm_concentration}, the squared (induced) 2-norm of random feature matrix $\Ab$ is bounded by
\begin{equation}
\|\Ab\|_2^2 = \lambda_{\max}(\Ab\Ab^*) = N \lambda_{\max}\left(\frac{1}{N}\Ab\Ab^*\right) \leq N(1+2\eta)
\end{equation}
with probability at least $1-\delta$ and thus
\begin{equation*}
\| f^* - f^\sharp \|_{L^2(D)}^2 \leq 2L_f^2h_X^2\vol(D) + Ch_X^d N(1+2\eta) \| \cb^* - \cb^\sharp\|_2^2
\end{equation*}
with probability at least $1-\delta$.

To estimate $\|\cb^*-\cb^\sharp\|^2_2$, we use the pseudo-inverse $\Ab^\dagger = \Ab^*(\Ab\Ab^*)^{-1}$ of $\Ab\in\C^{m\times N}$. Since the vector $\cb^\sharp = \Ab^\dagger \yb$, we apply the triangle inequality and the inequality $(a+b)^2 \leq 2a^2 + 2b^2$ to obtain
\begin{equation*}
\begin{aligned}
\| \cb^* - \cb^\sharp\|_2^2 \leq & 2\| \Ab^\dagger \Ab\cb^* - \Ab^\dagger \yb\|_2^2 + 2\|\Ab^\dagger \Ab \cb^* - \cb^*\|_2^2 \\
\leq & 2\| \Ab^\dagger \|_2^2 \|\Ab\cb^* - \yb\|_2^2 + 2\|\Ab^\dagger \Ab - \Ib\|_2^2\|\cb^*\|_2^2.
\end{aligned}
\end{equation*}
Using Theorem \ref{Thm_concentration}, the squared (induced) 2-norm of the pseudo-inverse $\Ab^\dagger$ is bounded by
\begin{equation*}
\|\Ab^\dagger\|_2^2 = \frac{1}{\lambda_{\min}(\Ab\Ab^*)} =  \frac{1}{N\lambda_{\min}\left(\frac{1}{N}\Ab\Ab^*\right)} \leq \frac{1}{N(1-2\eta)}
\end{equation*}
with probability at least $1-\delta$.
Note that $\Ab^\dagger \Ab - \Ib_m$ is an orthogonal projection, and hence $\|\Ab^\dagger \Ab - \Ib_m\|_2^2 \leq 1$. Moreover, we can show that 
\begin{equation*}
\begin{aligned}
\|\Ab\cb^* - \yb\|_2^2 =& \sum_{j=1}^m |f(\xb_j) - f^*(\xb_j)|^2 \leq m \sup_{\xb\in D} |f(\xb) - f^*(\xb)|^2 \\
\leq & m \left( \frac{2\|f\|_\rho^4}{T^2} + \frac{32T^2\log^2(2/\delta)}{N^2} + \frac{4\|f\|^2_\rho\log(2/\delta)}{N} \right)
\end{aligned}
\end{equation*}
where we use Theorem \ref{rf_term1} to obtain the last inequality. Putting Lemma \ref{lem_second_term}, Lemma \ref{decay_c} and all inequalities together gives, with probability at least $1-4\delta$,
\begin{equation}
 \|f^* - f^\sharp\|_{L^2(D)}^2 
\leq 2L_f^2h_X^2\vol(D) + Ch_X^d(1+2\eta)\left( \frac{2m}{1-2\eta}I_1 + 2I_2 \right)
\end{equation}
where 
\begin{equation*}
\begin{aligned}
&I_1 := \frac{2\|f\|_\rho^4}{T^2} + \frac{32T^2\log^2(2/\delta)}{N^2} + \frac{4\|f\|^2_\rho\log(2/\delta)}{N}, \\
&I_2 := \|f\|_\rho^2 + \frac{4T^2\log(2/\delta)}{N} + \sqrt{\frac{2T^2\|f\|_\rho^2\log(2/\delta)}{N}}.
\end{aligned}
\end{equation*}
\end{proof}

\begin{proof}[Proof of Theorem \ref{main_RF}]
Combining the results of Theorem \ref{rf_term1} and Theorem \ref{rf_term2} leads to an upper bound for the generalization error. Selecting $T=\|f\|_\rho \sqrt{N}$ yields the desired result. 
\end{proof}

\section{Error Analysis for Random Feature Operator Learning}
\label{Sec:Operator}
In this section, we present an error bound for our proposed random feature operator learning method. 
Let $G$ be the target operator, the generalization error of estimator $\Hat{G}$ is defined as $\|G(u) - \Hat{G}(u)\|_{L^2(D_\Vc)}$ for any $u\in\Uc$, where $\Uc$ is chosen to be the function space $\Fc(\rho)$ defined in \eqref{target_space}.
Using triangle inequality, we decompose the generalization error of our proposed estimator $\Hat{G} = R_\Vc\circ\Hat{f}\circ S_\Uc$ into 
\begin{equation*}
\begin{aligned}
& \|G(u) - R_\Vc\circ\Hat{f}\circ S_\Uc(u)\|_{L_2(D_\Vc)} \\
& \leq \underbrace{\|G(u) - R_\Vc\circ f\circ S_\Uc(u)\|_{L_2(D_\Vc)}}_{\text{Approximation Error}} + \underbrace{ \|R_\Vc\circ f\circ S_\Uc(u) - R_\Vc\circ\Hat{f}\circ S_\Uc(u)\|_{L_2(D_\Vc)} }_{\text{Estimation Error}}.
\end{aligned}
\end{equation*}

\begin{lemma}[{\bf Bound on Approximation Error}]
\label{operator_lemma1}
Assume that $G$ satisfies Assumption \ref{Assmps:Lipschitz_G} and map $f:\R^n\to\R^m$ is defined by $f=S_\Vc\circ G\circ R_\Uc$, then we have
\begin{equation*}
\|G(u) - R_\Vc\circ f\circ S_\Uc(u)\|_{L_2(D_\Vc)} \leq L_G \| u - R_\Uc\circ S_\Uc(u)\|_{L_2(D)} + \| v - R_\Vc \circ S_\Vc(v)\|_{L^2(D_\Vc)},
\end{equation*}
where $v = G\circ R_\Uc\circ S_\Uc(u)$. 
\end{lemma}
\begin{proof}
Recall the definition of map $f=S_\Vc\circ G\circ R_\Uc$ and apply triangle inequality to get
\begin{equation*}
\begin{aligned}
& \|G(u) - R_\Vc\circ f\circ S_\Uc(u)\|_{L_2(D_\Vc)} \\
\leq & \|G(u) - G\circ R_\Uc\circ S_\Uc(u)\|_{L_2(D_\Vc)} + \|G\circ R_\Uc\circ S_\Uc(u) - R_\Vc\circ S_\Vc\circ G\circ R_\Uc \circ S_\Uc(u)\|_{L_2(D_\Vc)} \\
\leq & L_G \| u - R_\Uc\circ S_\Uc(u)\|_{L_2(D)} + \| v - R_\Vc \circ S_\Vc(v)\|_{L^2(D_\Vc)}.
\end{aligned}
\end{equation*}
The last inequality holds since we assume $G:\Uc\to \Vc$ is Lipschitz continuous with Lipschitz constant $L_G$ and we denote $v=G\circ R_\Uc\circ S_\Uc(u)$.
\end{proof}

To bound the estimation error, we need to introduce the following vector-valued random feature map $\Bar{f}:\R^n\to\R^m$ where each $\Bar{f}_j:\R^n\to\R$ is a random feature approximation taking the form
\begin{equation}
\label{operator_temp}
\Bar{f}_j(\ub) = \sum_{k=1}^N \Bar{c}_k^{(j)} \exp(i\langle \omegab_k,\ub\rangle),
\end{equation}
where the coefficient vector $\Bar{\cb}^{(j)}\in\R^N$ for all $j\in[m]$ is trained by solving the min-norm interpolation problem
\begin{equation*}
\min_{\cb^{(j)}\in\R^N} \|\cb^{(j)}\|_2 \quad \mbox{ s.t. } \Bar{f}_j(\ub_\ell) = f_j(\ub_\ell) \quad \mbox{ for all } \ell\in[M].
\end{equation*}
Notice that we can never compute $\Bar{f}$ in practice since $f:\R^n\to\R^m$ is not known to us. 
\begin{assumption}
\label{Assmps:Lipschitz_f_bar}
We assume that each component $\Bar{f}_j:\R^n\to\R$ is Lipschitz continuous in the sense that there exists $\Bar{L}_j>0$ such that $|\Bar{f}_j(\ub_1) - \Bar{f}_j(\ub_2)| \leq \Bar{L}_j\|\ub_1 - \ub_2\|_2$ for any $\ub_1,\ub_2\in\R^n$.
\end{assumption}
Recall that $R_v:\R^m\to\Vc$ is a random feature recovery map and denote $f\circ S_\Uc(u)$, $\Hat{f}\circ S_\Uc(u)$, and $\Bar{f}\circ S_\Uc(u)$ by vectors $\zb$, $\Hat{\zb}$, and $\Bar{\zb}\in\R^m$, respectively. Then, we can write
\begin{equation*}
\begin{aligned}
R_\Vc (\zb) = \sum_{k=1}^N c_k \exp(i\langle \omegab_k,\yb\rangle), \hspace{0.05in} R_\Vc(\Hat{\zb}) = \sum_{k=1}^N \Hat{c_k} \exp(i\langle \omegab_k,\yb\rangle), \hspace{0.05in} R_\Vc(\Bar{\zb}) = \sum_{k=1}^N \Bar{c_k} \exp(i\langle \omegab_k,\yb\rangle),
\end{aligned}
\end{equation*}
where $\cb = \Ab^\dagger_y \zb$, $\Hat{\cb} = \Ab^\dagger_y\Hat{\zb}$, $\Bar{\cb} = \Ab^\dagger_y \Bar{\zb}\in\R^N$, and the matrix $\Ab_y$ is defined (component-wise) by $(\Ab_y)_{j,k} = \exp(i\langle \omegab_k, \yb_j\rangle)$ with $\Ab^\dagger_y=\Ab_y^*(\Ab_y\Ab_y^*)^{-1}$ being its pseudo-inverse. 

Let $\Uc$ and $\Vc$ are set of functions defined on compact domain $D_\Uc$ and $D_\Vc$, respectively. Distinct collocation points are denoted by $X=\{\xb_j\}_{j\in[n]}\subset D_\Uc$ and $Y=\{\yb_j\}_{j\in[m]}\subset D_\Vc$. The finite training data $\{ (u_j,v_j) \}_{j\in[M]} \subset \Uc \times \Vc$ satisfies $G(u_j) = v_j$ for all $j\in[M]$ and is accessible via sampling operators $S_\Uc$ and $S_\Vc$, i.e., denote the point evaluations of function $u_j$ and $v_j$ by
\begin{equation*}
\ub_j = S_\Uc(u_j) = [u_j(\xb_1), \dots, u_j(\xb_n)]^\top \in\R^n, \quad \vb_j = S_\Vc(v_j) = [v_j(\yb_1), \dots, v_j(\yb_m)]^\top \in\R^m.
\end{equation*}
Let $U\subset \R^n$ be a compact set with Lipschitz boundary and $\{ \ub_j\}_{j\in[M]}$ be the set of all point evaluation vectors with fill-in distance 
\begin{equation*}
h_\ub := \max_{\ub'\in U}\min_{1\leq j\leq M} \|\ub_j - \ub'\|_2.
\end{equation*}

\begin{lemma}[{\bf Bound on Estimation Error}]
\label{operator_lemma2}
Suppose that target operator $G$ satisfies Assumption \ref{Assmps:Lipschitz_G}. Let $f$, $\hat{f}$, and $\Bar{f}$ be defined by \eqref{operator_f}, \eqref{operator_rf}, and  \eqref{operator_temp}, respectively. Suppose that $f$ and $\Bar{f}$ satisfy Assumption \ref{Assmps:f} and \ref{Assmps:Lipschitz_f_bar}, respectively. For some parameter $\eta\in(0,1)$, then we have
\begin{equation*}
\begin{aligned}
& \|R_\Vc \circ f \circ S_\Uc(u) - R_\Vc \circ \Hat{f} \circ S_\Uc(u) \|_{L_2(D_\Vc)} \\
\leq & \sqrt{\frac{\vol(D_\Vc)}{1-\eta}} \left( 2\sqrt{m}h_\ub\sup_{j\in[m]}(L_j+\Bar{L}_j) + \frac{\sqrt{2}\|S_\Vc\| L_G\sqrt{M}}{\sqrt{1-\eta}} \sup_{\ell\in[M]}\| u_\ell - R_\Uc\circ S_\Uc(u_\ell)\|_{L_2(D_\Uc)} \right)
\end{aligned}
\end{equation*}
with probability at least $1-\delta$. 
\end{lemma}
\begin{proof}
Following the notations of random feature functions, we can bound the error by
\begin{equation}
\label{boundI2}
\begin{aligned}
& \left\|R_\Vc\circ f\circ S_\Uc(u) - R_\Vc\circ\Hat{f}\circ S_\Uc(u)\right\|^2_{L_2(D_\Vc)} \leq \int_{D_\Vc} \left|\sum_{k=1}^N \left(c_k - \Hat{c_k}\right) \exp(i\langle \omegab_k,\yb\rangle) \right|^2 d\yb \\
\leq & \int_{D_\Vc} N \|\cb-\Hat{\cb}\|^2_2 d\yb 
\leq N\vol(D_\Vc)\|\Ab_y^\dagger\|_2^2 \|\zb-\Hat{\zb}\|_2^2 
\leq \frac{\vol(D_\Vc)}{1-\eta}\|\zb-\Hat{\zb}\|_2^2,
\end{aligned}
\end{equation}
where we use the bound on the (induced) 2-norm of the pseudo-inverse $\Ab_y^\dagger$.
It remains to bound $\|\zb-\Hat{\zb}\|_2$. By definitions, we write its square as
\begin{equation*}
\|\zb-\Hat{\zb}\|_2^2 = \sum_{j=1}^m |f_j(\ub) - \Hat{f}_j(\ub)|_2^2 \leq 2\sum_{j=1}^m |f_j(\ub) - \Bar{f}_j(\ub)|^2 + 2\sum_{j=1}^m |\Bar{f}_j(\ub) - \Hat{f}_j(\ub)|^2.
\end{equation*}
Since $f_j$ and $\Bar{f}_j$ are Lipschitz continuous with Lipschitz constant $L_j, \Bar{L}_j$ for each $j\in[m]$, and $\Bar{f}_j(\ub_\ell) = f_j(\ub_\ell)$ for each training samples $\ub_\ell$. Then we can give the following bound
\begin{equation*}
\begin{aligned}
& |f_j(\ub) - \Bar{f}_j(\ub)| \leq |f_j(\ub) - f_j(\ub_\ell)| + |f_j(\ub_\ell) - \Hat{f}_j(\ub_\ell)| + |\Hat{f}_j(\ub_\ell) - \Hat{f}_j(\ub)|   \\
& \leq L_j \|\ub-\ub_\ell\|_2 + \Bar{L}_j\|\ub-\ub_\ell\|_2 \\
& \leq  (L_j + \Bar{L}_j) h_\ub,
\end{aligned}
\end{equation*}
where we assume that $\ub_\ell$ is one training samples which is closest to $\ub$, and hence the Euclidean distance $\|\ub-\ub_\ell\|_2$ is less than the fill-in distance $h_\ub$, i.e $\|\ub-\ub_\ell\|_2 \leq h_\ub$.
Recall the constructions of $\Bar{f}$ and $\Hat{f}$, then we have
\begin{equation*}
\begin{aligned}
& \sum_{j=1}^m \left|\Bar{f}_j(\ub) - \Hat{f}_j(\ub)\right|^2 \leq \sum_{j=1}^m \left( \sum_{k=1}^N \left|\Bar{\cb}_k^{(j)} - \Hat{\cb}_k^{(j)}\right|\right)^2 \leq \sum_{j=1}^m N\|\Bar{\cb}^{(j)} - \Hat{\cb}^{(j)}\|_2^2 \\
\leq & N\|\Ab^\dagger\|_2^2 \sum_{j=1}^m \| \Bar{\vb}^{(j)} - \Hat{\vb}^{(j)}\|^2_2 \\
\leq & \frac{1}{1-\eta} \sum_{\ell=1}^M \|S_\Vc\circ G \circ R_\Uc \circ S_\Uc(u_\ell) -   S_\Vc\circ G \circ u_\ell\|^2_2 \\
\leq & \frac{1}{1-\eta} \|S_\Vc\|^2 L_G^2 M \sup_{\ell\in[M]}\| u_\ell - R_\Uc\circ S_\Uc(u_\ell)\|^2_{L_2(D_\Uc)}.
\end{aligned}
\end{equation*}
Therefore, we can bound $\|\zb-\Hat{\zb}\|_2^2$ by
\begin{equation}
\label{bound_z}
\| \zb - \Hat{\zb} \|_2^2 \leq 4mh_\ub^2\sup_{j\in[m]}(L_j^2+\Bar{L}_j^2) + \frac{2\|S_\Vc\|^2 L^2_G M}{1-\eta} \sup_{\ell\in[M]}\| u_\ell - R_\Uc\circ S_\Uc(u_\ell)\|^2_{L_2(D_\Uc)}.
\end{equation}
Substituting \eqref{bound_z} into \eqref{boundI2}, and then taking the squared root lead to the desired error bound.
\end{proof}

\begin{theorem}[{\bf Generalization Error Bound for Operator Learning}]
\label{OL_error}
Suppose that target operator $G$ satisfies Assumption \ref{Assmps:Lipschitz_G}. Let $f$, $\hat{f}$, and $\Bar{f}$ be defined by \eqref{operator_f}, \eqref{operator_rf}, and  \eqref{operator_temp}, respectively. Suppose that $f$ and $\Bar{f}$ satisfy Assumption \ref{Assmps:f} and \ref{Assmps:Lipschitz_f_bar}, respectively.
Then, for any $u\in B_1(\Fc(\rho))$, it holds that
\begin{equation*}
\begin{aligned}
& \|G(u) - R_\Vc\circ\Hat{f}\circ S_\Uc(u)\|_{L_2(D_\Vc)} \leq L_G \| u - R_\Uc\circ S_\Uc(u)\|_{L_2(D)} + \| v - R_\Vc \circ S_\Vc(v)\|_{L^2(D_\Vc)} + \\ 
& \sqrt{\frac{\vol(D_\Vc)}{1-\eta}} \left( 2\sqrt{m}h_\ub\sup_{j\in[m]}(L_j+\Bar{L}_j) + \frac{\sqrt{2}\|S_\Vc\| L_G\sqrt{M}}{\sqrt{1-\eta}} \sup_{k\in[M]}\| u_k - R_\Uc\circ S_\Uc(u_k)\|_{L_2(D_\Uc)} \right)
\end{aligned}
\end{equation*}
with probability at least $1-\delta$.
Furthermore, let training inputs satisfy $u_j\in B_1(\Fc(\rho))$ for all $j\in[M]$.
Suppose the collocation points $X=\{\xb_j\}_{j\in[n]}$ and $Y=\{\yb_j\}_{j\in[m]}$, random features $\{\omegab_k\}_{k\in[N]}$, and random feature matrix $\Ab\in\C^{m\times N}$ satisfy Assumption \ref{Assmps:data}. In addition, the conditions in Theorem \ref{main_RF} hold.
Then, for any $u\in B_1(\Fc(\rho))$, it holds that
\begin{equation*}
\|G(u) - R_v\circ\Hat{f}\circ S_u(u)\|_{L_2(D_\Vc)} \leq \mathcal{O}\left(\frac{1}{\sqrt{N}} + h_X + h_X^{d/2} + h_Y + h_Y^{d/2} + h_\ub \right)
\end{equation*}
with probability at least $1-6\delta$.
\end{theorem}
\begin{proof}
We first decompose the generalization error into two terms
\begin{equation*}
\begin{aligned}
& \|G(u) - R_\Vc\circ\Hat{f}\circ S_\Uc(u)\|_{L_2(D_\Vc)} \\
\leq& \underbrace{\|G(u) - R_\Vc\circ f\circ S_\Uc(u)\|_{L_2(D_\Vc)}}_{\text{Approximation Error}} + \underbrace{ \|R_\Vc\circ f\circ S_\Uc(u) - R_\Vc\circ\Hat{f}\circ S_\Uc(u)\|_{L_2(D_\Vc)} }_{\text{Estimation Error}}.
\end{aligned}
\end{equation*}
Bounds on approximation error and estimation error are given by Lemma \ref{operator_lemma1} and Lemma \ref{operator_lemma2}, respectively. 
Applying the generalization error bound for random feature model in Theorem \ref{main_RF} to $u$, $G\circ R_\Uc \circ S_\Uc(u)$, and $u_k$ leads to the desired result.
\end{proof}

\section{Numerical Experiments}
\label{Sec:experiment}
Numerical experiments are used to benchmark and compare our proposed random feature operator learning method to kernel-based method \cite{BATLLE2024kernel} and neural operator methods, including DeepONet \cite{DeepONet2021Lu} and FNO \cite{Li2020FourierNO}.
We measure the model performance by using an empirical relative error on a test set $U_{test}$: 
\begin{equation*}
\mathrm{Error}(\Hat{G}) = \frac{1}{|U_{test}|} \sum_{u_k\in U_{test}} \frac{\|G(u_k)-\hat{G}(u_k)\|_{L^2(D_\Vc)}}{\|G(u_k)\|_{L^2(D_\Vc)}}
\end{equation*}
where $G:\Uc\to\Vc$ is the true operator. For any $v\in\Vc$, we take $\|v\|_{L^2(D_\Vc)} = \sqrt{ \int_{D_\Vc} |v(x)|^2 dx}$, which in turn is estimated by the function values at collocation points. 
We outline the setup of each problem in the following subsections. In all problems, $\Uc$ and $\Vc$ are spaces of real-valued functions defined on domains $D_\Uc, D_\Vc\subset\R^d$ for $d=1,2$. 
All the experiments are conducted using Python and our source codes are available online\footnote{\url{https://github.com/liaochunyang/RandomFeatureOperatorLearning}}.

For the kernel method, we consider RBF kernel and Mat{\'e}rn kernel. Specifically, the RBF kernel function with parameter $\gamma>0$ is $K(\xb,\xb') = \exp(-\gamma\|\xb- \xb'\|_2^2)$. The Mat{\' e}rn kernel with scaling parameter $\sigma>0$ and smoothness parameter $\nu>0$ is defined by 
$$K_{\nu, \sigma}(\xb,\xb') = \frac{2^{1-\nu}}{\Gamma(\nu)} \left( \frac{\sqrt{2\nu}\|\xb-\xb'\|_2}{\sigma} \right)^\nu K_\nu\left( \frac{\sqrt{2\nu}\|\xb-\xb'\|_2}{\sigma} \right),$$
where $\Gamma$ is the gamma function, and $K_\nu$ is the modified Bessel function of the second kind of order $\nu$.
The Gaussian random feature model with parameter $\gamma>0$ means that the random features $\omegab$ are drawn from normal distribution $\Nc(0,2\gamma)$, which approximates RBF kernel with parameter $\gamma$. The Cauchy random feature model with parameter $\gamma>0$ means that the random features $\omegab$ are drawn from tensor-product Cauchy distribution with scaling parameter $\gamma>0$.

\subsection{Advection Equations \Romannum{1}, \Romannum{2}, and \Romannum{3}}
Consider the one-dimensional advection equation:
\begin{equation}
\begin{aligned}
    \frac{\partial v}{\partial t} + \frac{\partial v}{\partial x} &= 0 \quad && x\in(0,1), t\in(0,1] \\
    v(x,0) &= u(x) \quad && x\in(0,1), 
\end{aligned}
\end{equation}
we aim to learn the operator $G$ mapping from the initial condition $u(x)=v(x, 0)$ to the solution at time $t=0.5$, which is denoted by $v(\cdot, 0.5)$.
The domains of function spaces $\Uc$ and $\Vc$ are $D_\Uc = D_\Vc = (0,1)$.
This problem was considered in \cite{LU2022operator, dehoop2022costaccuracy} with different distributions $\mu$ for the initial condition. We consider, referred as Advection \Romannum{1} here, the initial condition which is a square wave centered at $x = c$ of width $b$ and height $h$:
$$ u(x) = h1_{ \{c- 2b ,c + 2b\} }.$$
The parameters $(c, b, h)$ are randomly generated following the uniform distribution on $[0.3, 0.7] \times [0.3, 06] \times [1, 2]$. The spatial grid is of resolution 40, and we use 1000 instances for training and 200 instances for testing. Figure \ref{Fig:Advection1} shows an example of training input, training output, true test example and its Cauchy random feature approximation, and the corresponding pointwise error. 
\begin{figure}[!htbp]
\centering 
\subfigure[Training input]{\includegraphics[width=36mm]{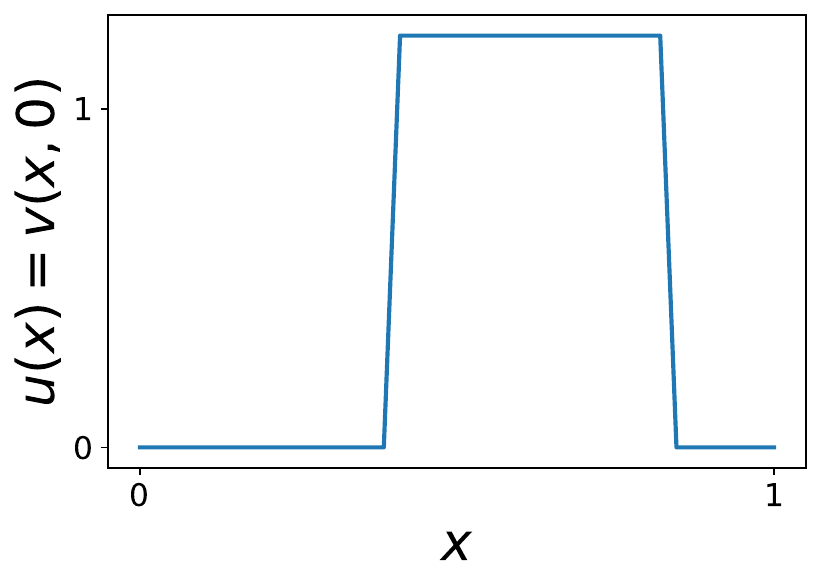}}
\subfigure[Training output]{\includegraphics[width=36mm]{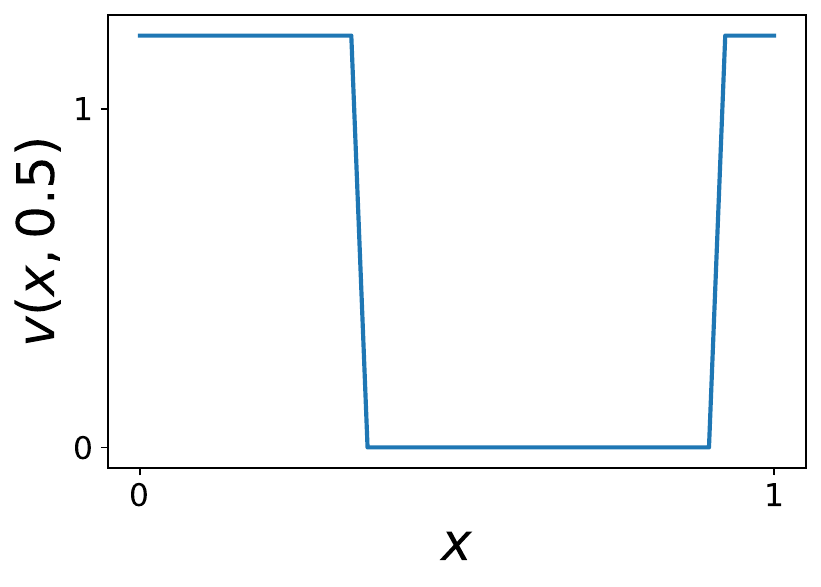}}
\subfigure[True and predicted test]{\includegraphics[width=36mm]{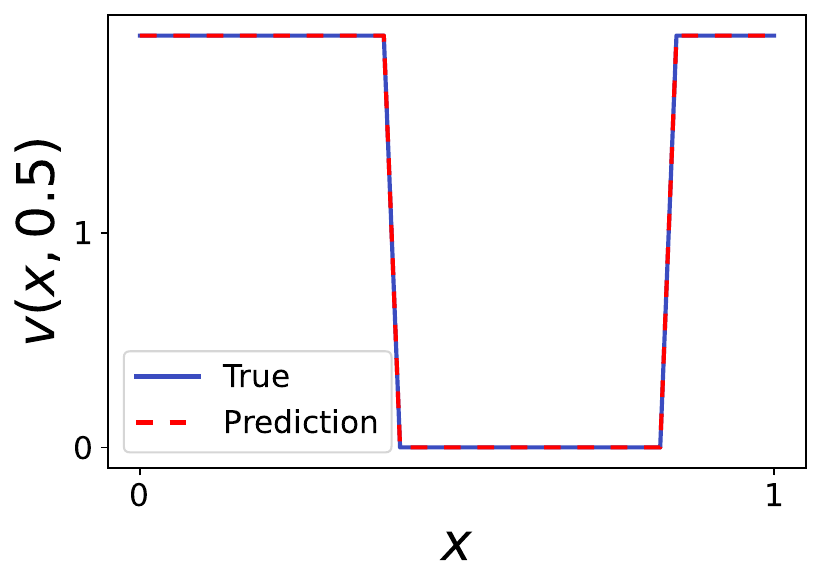}}
\subfigure[Pointwise error]{\includegraphics[width=36mm]{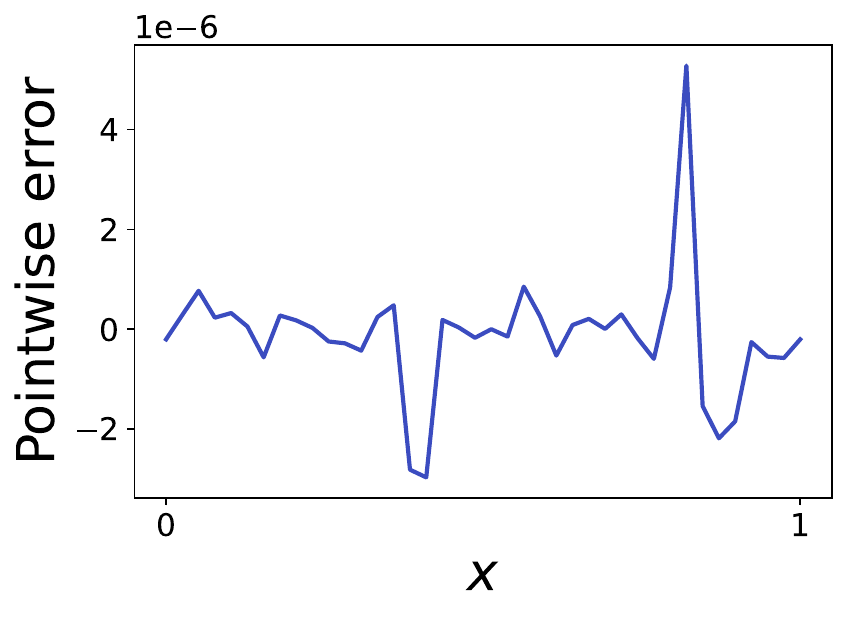}}
\caption{An example of training input, training output, test function and its Cauchy random feature approximation, and pointwise error for the Advection Equation \Romannum{1}.}
\label{Fig:Advection1}
\end{figure}
Advection equation \Romannum{2} takes a more complicated initial condition, i.e. 
$$ u(x) = h_1 1_{ \{c_1 - w, c_1 + w\} } + \sqrt{ \max(h_2^2 - a_2^2(x - c_2)^2, 0)}.$$
The resolution is of 40 and we use 1000 samples to train the random feature models and another 1000 samples to test the performance.
Similarly, we show an example of training input, training output, true test example, its Cauchy random feature approximation, and the corresponding pointwise error in Figure \ref{Fig:Advection2}.

\begin{figure}[!htbp]
\centering     
\subfigure[Training input]{\includegraphics[width=36mm]{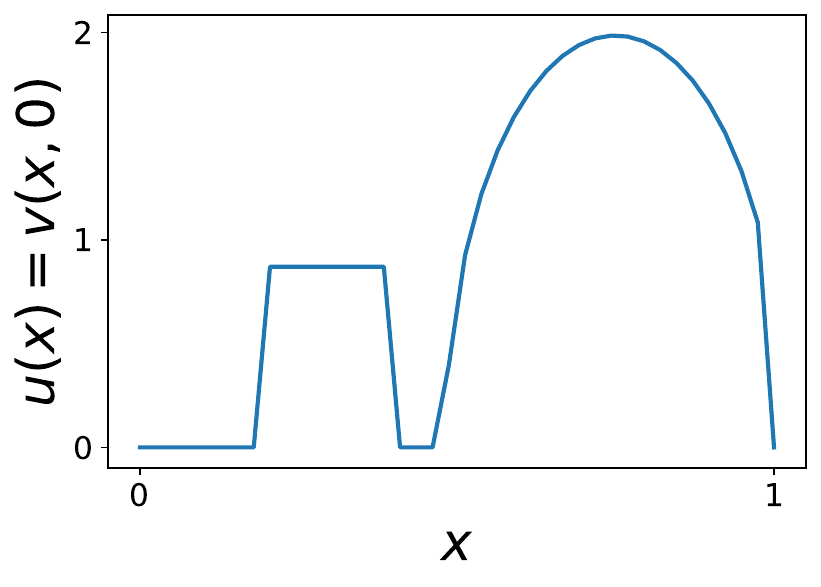}}
\subfigure[Training output]{\includegraphics[width=36mm]{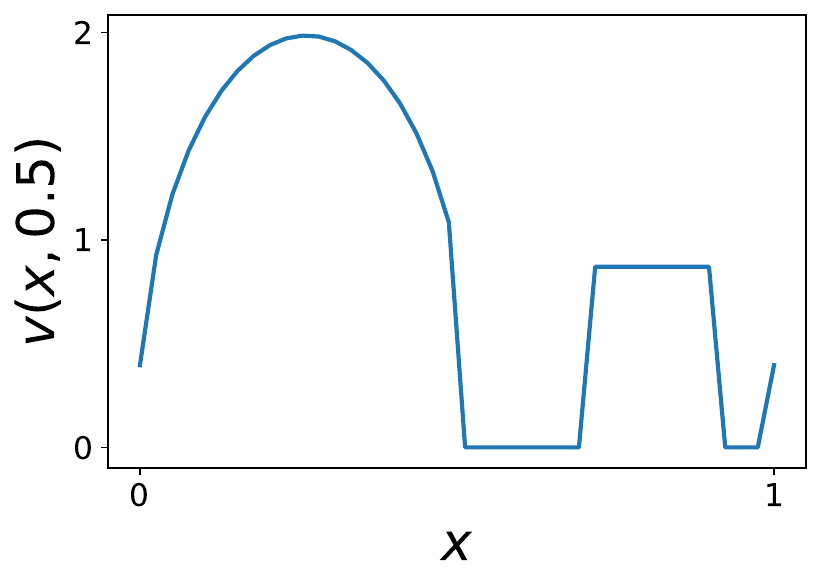}}
\subfigure[True and predicted test]{\includegraphics[width=36mm]{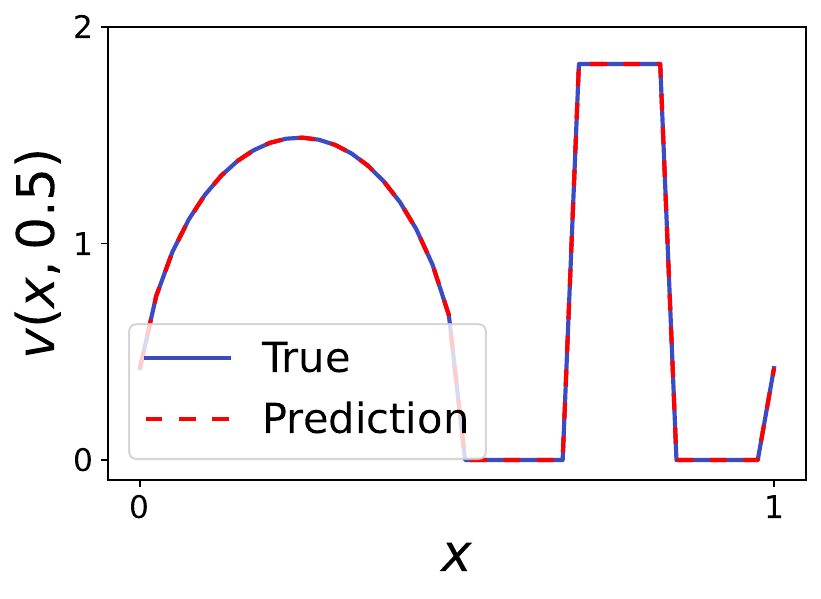}}
\subfigure[Pointwise error]{\includegraphics[width=36mm]{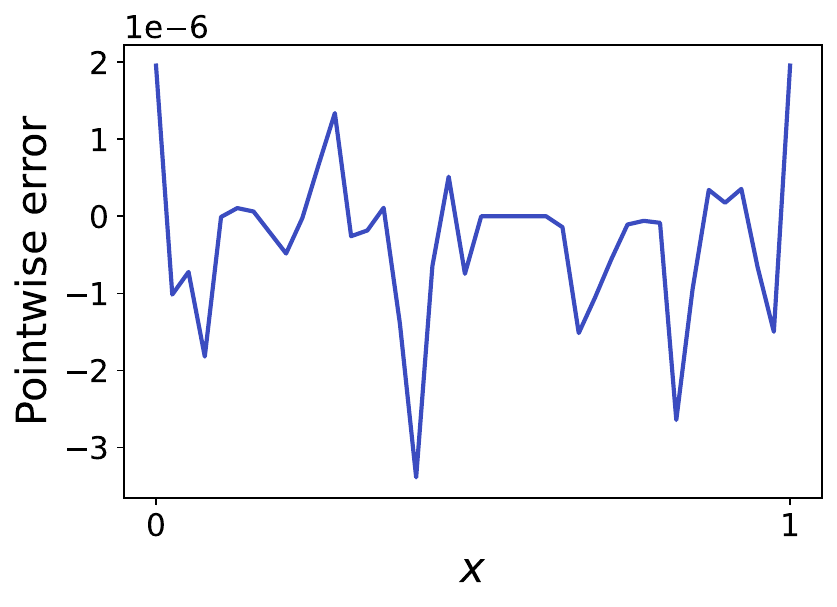}}
\caption{An example of training input, training output, test function and its Cauchy random feature approximation, and pointwise error for the Advection Equation \Romannum{2}.}
\label{Fig:Advection2}
\end{figure}

A different initial condition was considered in \cite{dehoop2022costaccuracy}. We call it Advection equation \Romannum{3} here. Suppose that $\tilde{u}_0$ is generated from a Gaussian process $GP(0,(-\Delta+32\Ib)^{-2})$ where $\Delta$ and $\Ib$ represent the Laplacian and the identity, respectively. 
The initial condition is defined as 
$$ u=-1+21_{ \{\tilde{u}_0 \geq 0\} }. $$
The resolution is of 200 and we use 1000 samples for training and another 1000 instances for testing.
In Figure \ref{Fig:Advection3}, we show an example of training input and output, true test function, its Cauchy random feature approximation, and the corresponding pointwise error.

\begin{figure}[!htbp]
\centering 
\subfigure[Training input]{\includegraphics[width=36mm]{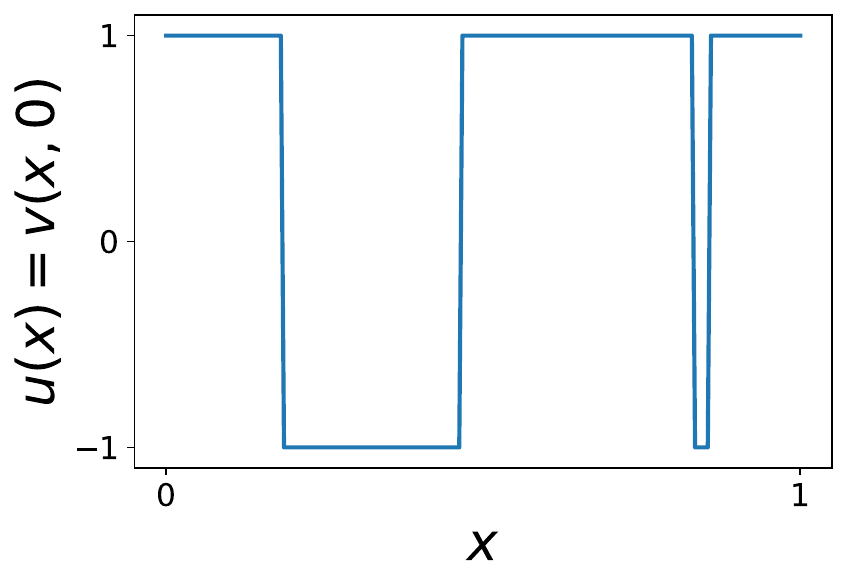}}
\subfigure[Training output]{\includegraphics[width=36mm]{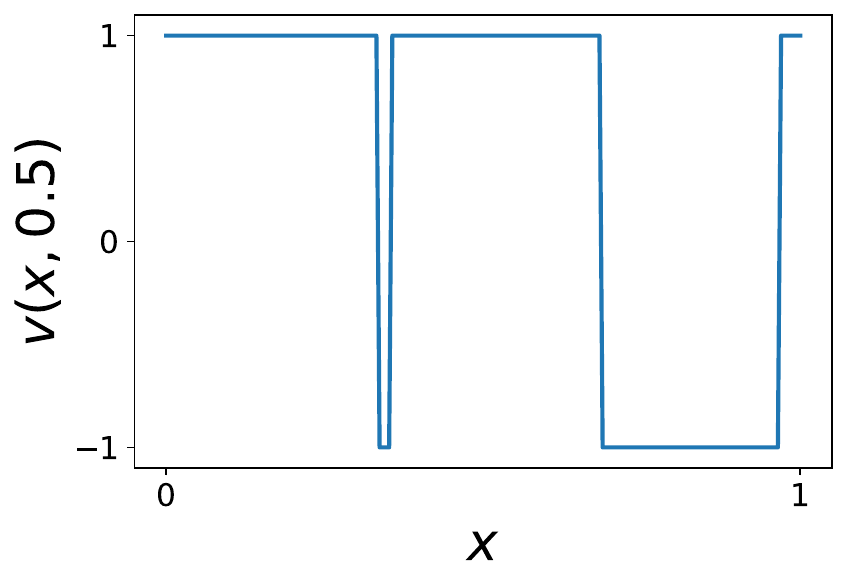}}
\subfigure[True and predicted test]{\includegraphics[width=35mm]{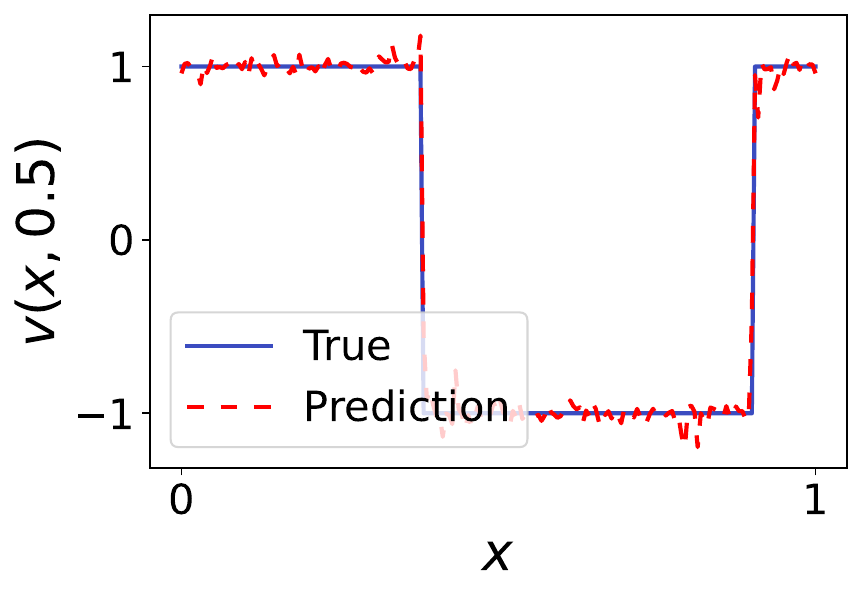}}
\subfigure[Pointwise error]{\includegraphics[width=35mm]{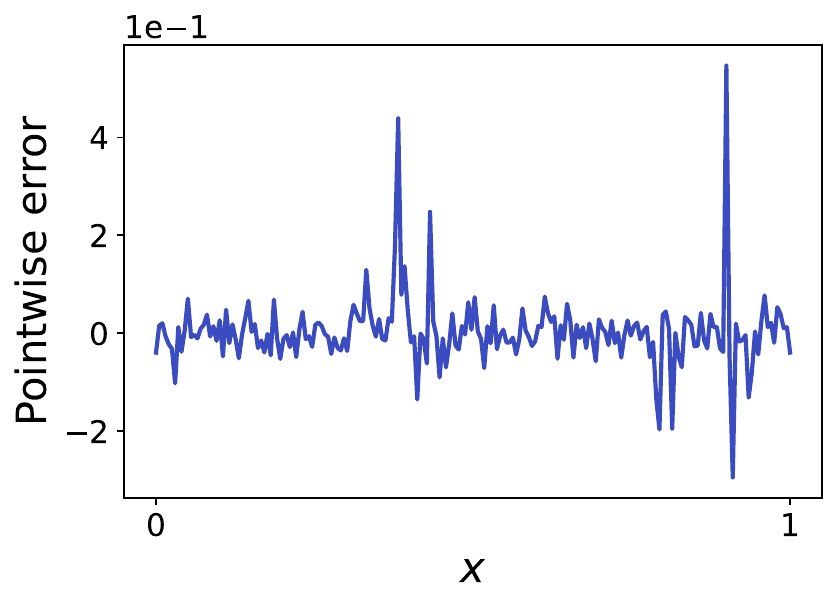}}
\caption{An example of training input, training output, test function, its Cauchy random feature approximation, and pointwise error for the Advection Equation \Romannum{3}.}
\label{Fig:Advection3}
\end{figure}

In Table \ref{Table:Advection}, we compare the relative test errors and computational times of RBF kernel-based model, Gaussian random feature model, and Cauchy random feature model. The scaling parameter $\gamma$ for each model and the number of features $N$ for random feature models are included as well. 
We observe that our proposed random feature method is reliable in terms of accuracy and even outperforms vanilla kernel method. Moreover, the training time is reduced significantly.
As discussed in \cite{BATLLE2024kernel}, since the Advection \Romannum{3} tests have more jumps, the prediction task poses challenges for smooth models such as our model and the kernel models around the discontinuities.

\begin{table}[!htbp]
\centering
\begin{tabular}{|c|c|c|c|c|}
\hline
\multirow{2}{*}{} & & & & relative \\ 
& model & $N$ & $\gamma$ & test error \\ \hline
\multirow{3}{*}{Advection \Romannum{1}}   & RBF Kernel & n/a & 0.5 & $2.08\times 10^{-5}$ \%   \\ \cline{2-5}  
& Random Feature (Gaussian) & 5000 & $10^{-5}$ & $4.70 \times 10^{-6}$ \%  \\ \cline{2-5}
& Random Feature (Cauchy) & 5000 & $10^{-5}$ & $1.26 \times 10^{-6}$ \%  \\ \hline
\multirow{3}{*}{Advection \Romannum{2}}  & RBF Kernel & n/a & 0.5 &  $4.20\times 10^{-5}$ \% \\ \cline{2-5} 
 & Random Feature (Gaussian) & 5000 & $10^{-5}$ & $6.26 \times 10^{-6}$ \% \\ \cline{2-5}
 & Random Feature (Cauchy) & 5000 & $10^{-5}$ & $2.16\times10^{-6}\%$           \\ \hline
 \multirow{3}{*}{Advection \Romannum{3}} & RBF Kernel & n/a & 0.5 & 0.17  \\ \cline{2-5} 
 & Random Feature (Gaussian) & 5000 & 0.01 & 0.22 \\ \cline{2-5}
 & Random Feature (Cauchy) & 3000 & $10^{-6}$ & 0.14  \\ \hline
\end{tabular}
\caption{Summary of numerical results of Advection Equations: we report relative test errors for proposed random feature methods and compare with kernel method proposed in \cite{BATLLE2024kernel}. We also report the number of features $N$ and the scaling parameter $\gamma$ for each experiment.}
\label{Table:Advection}
\end{table}

\subsection{Burgers’ Equation}
Consider the one-dimensional Burgers’ equation: 
\begin{equation}
\begin{aligned}
\frac{\partial w}{\partial t} + w\frac{\partial w}{\partial x} = \nu \frac{\partial^2 w}{\partial x^2},& \quad && (x,t)\in(0,1)\times (0,1] \\
w(x,0) = u(x),& \quad && x \in (0,1).
\end{aligned}   
\end{equation}
We set the viscosity parameter $\nu=0.1$. 
Our goal is to learn the operator $G:u(x)\mapsto w(x,1)$, which maps the initial condition $u(x) = w(x,0)$ to the solution $w(x,t)$ at time $t=1$.
The data are generated by sampling the initial condition $u$ from a Gaussian Process $\mu\sim GP(0, 625(-\Delta + 25I)^{-2}))$, where $\Delta$ and $I$ represent the Laplacian and the identity, respectively. 
As in \cite{LU2022operator, BATLLE2024kernel}, we use a spatial resolution with 128 grid points to represent the input and output functions, and use 1848 instances for training and 200 instances for testing.
Figure \ref{Fig:Burgers} shows an example of training input, training output, true test function and its approximation using Cauchy random feature model along with the pointwise error.  

\begin{figure}[!htbp]
\centering     
\subfigure[Training input]{\includegraphics[width=36mm]{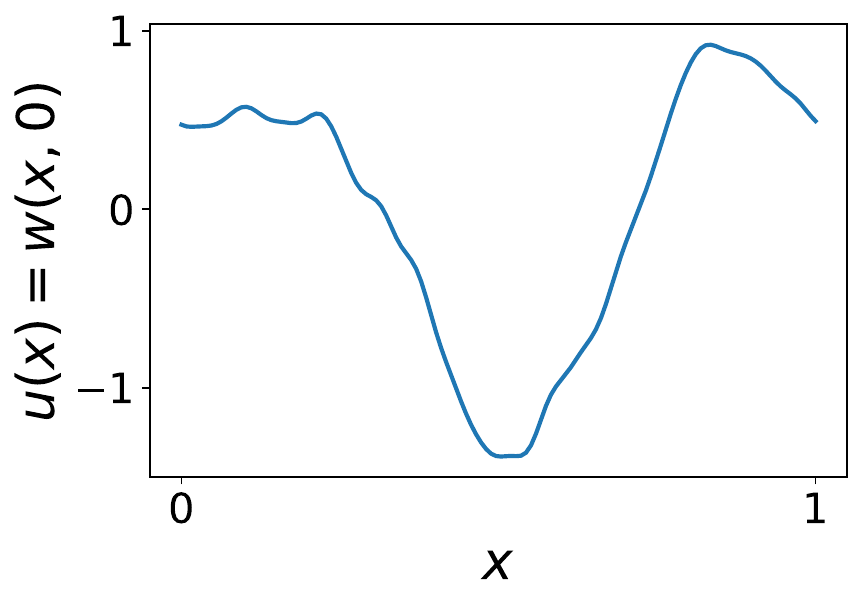}}
\subfigure[Training output]{\includegraphics[width=36mm]{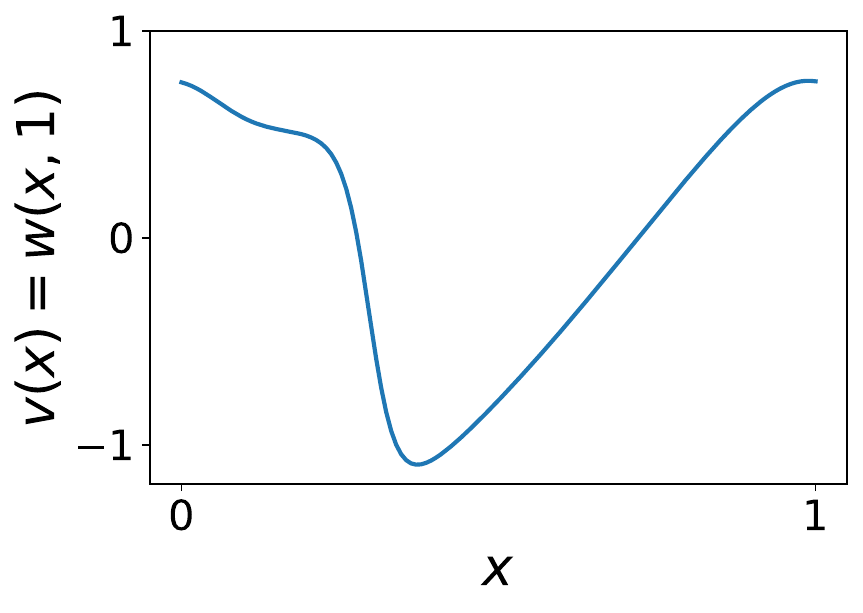}}
\subfigure[True test and its prediction]{\includegraphics[width=36mm]{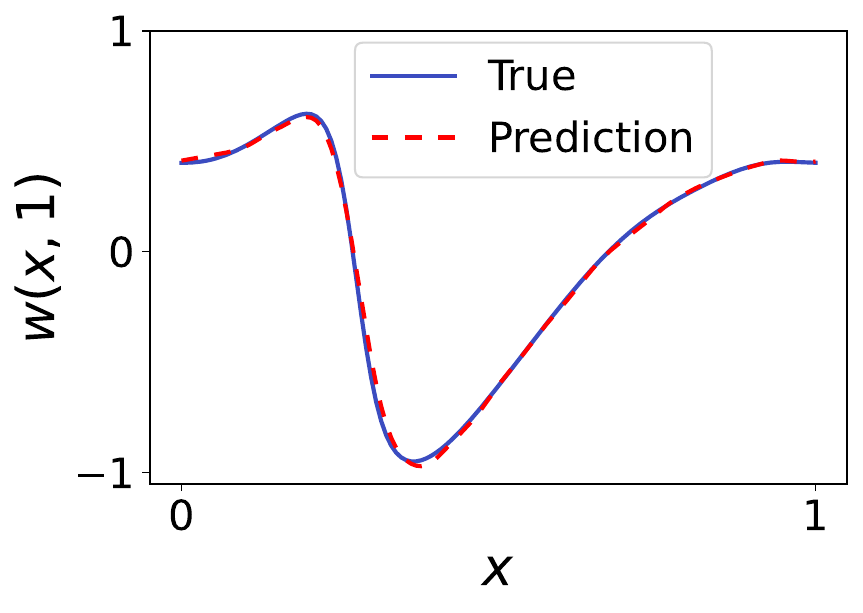}}
\subfigure[Pointwise error]{\includegraphics[width=36mm]{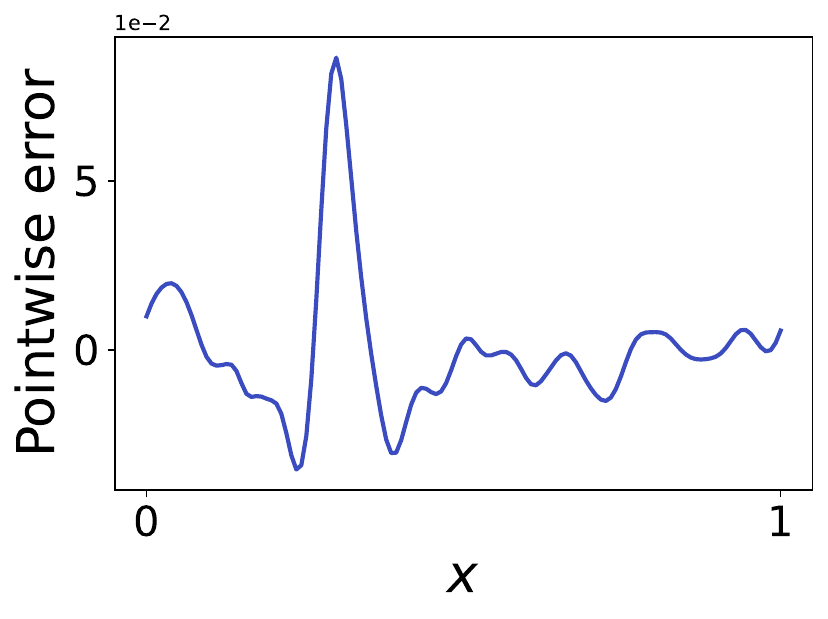}}
\caption{Burgers' equation: an example of (a) training input, (b) training output, (c) test function and its prediction using our random feature method, and (d) pointwise error. We randomly generate $N=10^5$ i.i.d samples from the tensor-product Cauchy distribution with scaling parameter $\gamma=0.01$ as random features.}
\label{Fig:Burgers}
\end{figure}

\subsection{Darcy Problem}
Consider the two-dimensional Darcy flow problem:
\begin{equation}
\begin{aligned}
    -\mathop{\mathrm{div}} e^u \nabla v &= w,  && \mbox{ in }  D, \\
    v &= 0, && \mbox{ on } \partial D
\end{aligned}
\end{equation}
with $D=(0,1)^2$ and zero Dirichlet boundary condition,
we are interested in learning the operator $G:u\mapsto v$ where $u$ is the permeability field and $v$ is the solution. Hence, the domains of function $u$ and $v$ are the same, i.e., $D_\Uc = D_\Vc = (0,1)^2$. 
The source term $w$ is assumed to be fixed. 
Here, we consider the piecewise constant permeability field. Specifically, the coefficient $u$ is sampled from $u \sim \log \circ h (\mu)$, where $\mu$ is sampled from Gaussian Process $\mu \sim GP(0, (-\Delta+9I)^{-2})$ and $h$ is binary function mapping positive inputs to 12 and negative inputs to 3.
Therefore, the permeability field $e^u$ is piecewise constant, see an example of permeability field (training input) and the corresponding solution (training output) in Figure \ref{Fig:Darcy}. 
The grid of resolution is $29\times29$ and we use 1000 samples for training, and 200 samples for testing. 
For our random feature model, we randomly generate $N=10^4$ random features from the tensor-product Cauchy distribution with scaling parameter $\gamma=2\times10^{-4}$.
Figure \ref{Fig:Darcy} shows one example of true test solution $v$, its Cauchy random feature approximation, and the pointwise error. 

\begin{figure}[!htbp]
\centering
\subfigure[Training input]{\includegraphics[width=28mm]{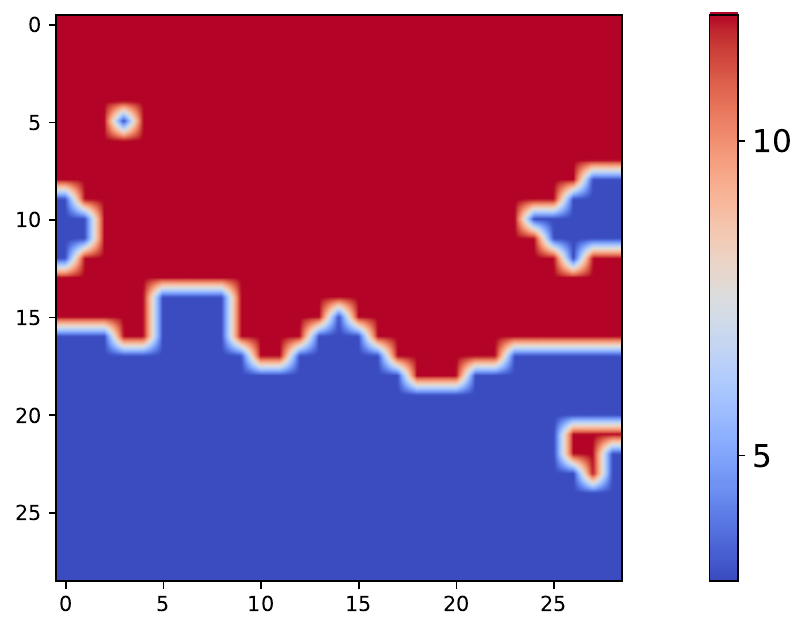}}
\subfigure[Training output]{\includegraphics[width=28mm]{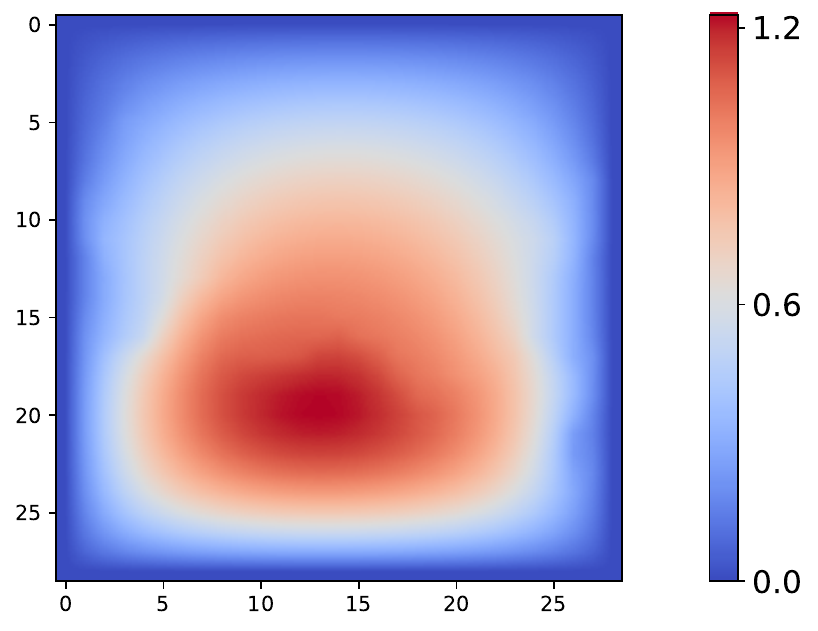}}
\subfigure[True test]{\includegraphics[width=28mm]{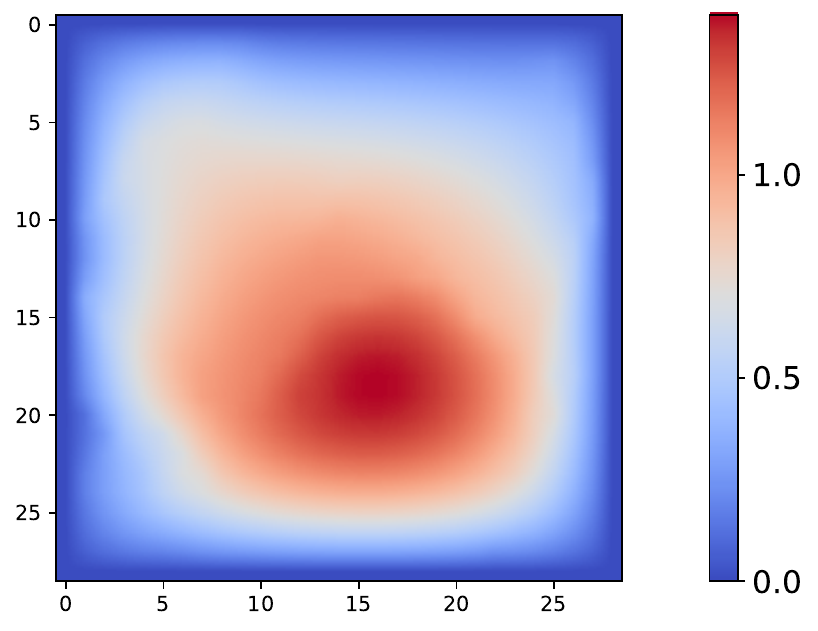}}
\subfigure[Predicted test]{\includegraphics[width=28mm]{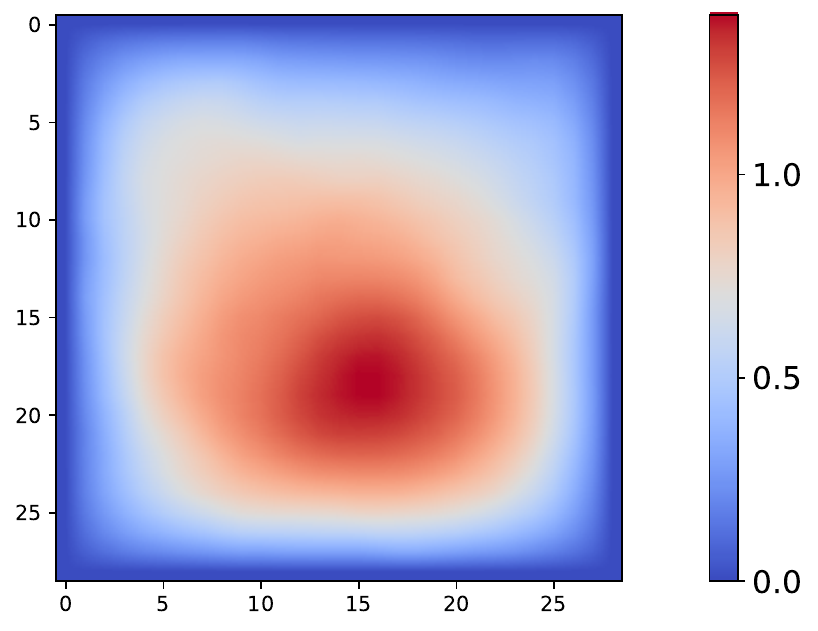}}
\subfigure[Pointwise error]{\includegraphics[width=28mm]{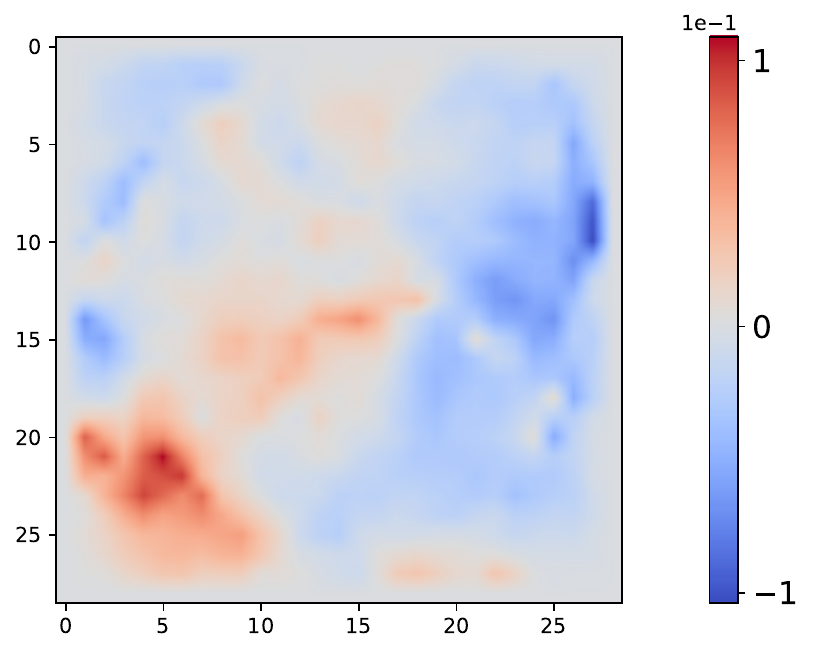}}
\caption{An example of training input, training output, true test and its Cauchy random feature approximation, and the pointwise error for Darcy problem in a rectangular domain with piecewise constant permeability field.}
\label{Fig:Darcy}
\end{figure}

\subsection{Helmholtz's Equation}
 We consider the Helmholtz PDE
\begin{equation}
\label{Helmholtz}
\begin{aligned}
& \left( - \Delta - \frac{\omega^2}{u^2(x)} \right)v = 0, \quad x\in(0,1)^2 \\
& \frac{\partial v}{\partial n} = 0, \quad x\in\{0,1\}\times[0,1]\cup[0,1]\times\{0\} \quad \mbox{ and } \quad \frac{\partial v}{\partial n} = v_N, \quad x\in[0,1]\times\{1\}, 
\end{aligned}
\end{equation}
where frequency $\omega$ and wave speed field $u: D_\Uc \to \R$ are assumed to be fixed and given. The excitation field $v: D_\Vc \to \R$ solves equation \eqref{Helmholtz}. 
The domains of function $u$ and $v$ are $D_\Uc = D_\Vc = (0,1)^2$ in this example.
In the numerical experiments, we take $\omega=10^3$ and $v_N=I_{\{0.35\leq x\leq 0.65\}}$. 
Our goal is to learn operator $G:u\mapsto v$ from the wave speed field $u$ to the excitation field $v$. 
The wave speed field $u$ takes the form $u(x) = 20+\tanh(\Tilde{u}(x))$, where $\Tilde{u}$ is drawn from Gaussian Process $GP(0,(-\Delta+9I)^{-2})$. Here, denote $\Delta$ by the Laplacian on $D_\Uc$ subject to homogeneous Neumann boundary conditions on the space of spatial-mean zero functions. 
As described in \cite{BATLLE2024kernel,dehoop2022costaccuracy}, samples are generated by solving the equation \eqref{Helmholtz} using a Finite Element Method on a discretization grid of size $26\times26$ of the unit square. 
We use 1000 functions for training and 1000 functions for testing.
In Figure \ref{Fig: Helmholtz}, we show an example of training input, training output, true test function, Cauchy random feature approximation, and the corresponding pointwise error. 
\begin{figure}[!htbp]
\centering     
\subfigure[Training input]{\includegraphics[width=28mm]{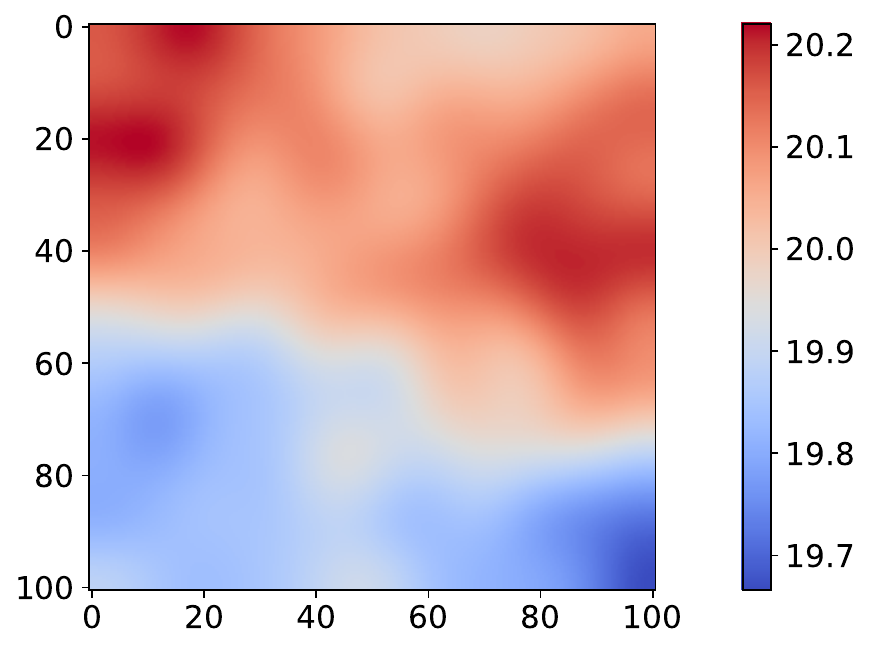}}
\subfigure[Training output]{\includegraphics[width=28mm]{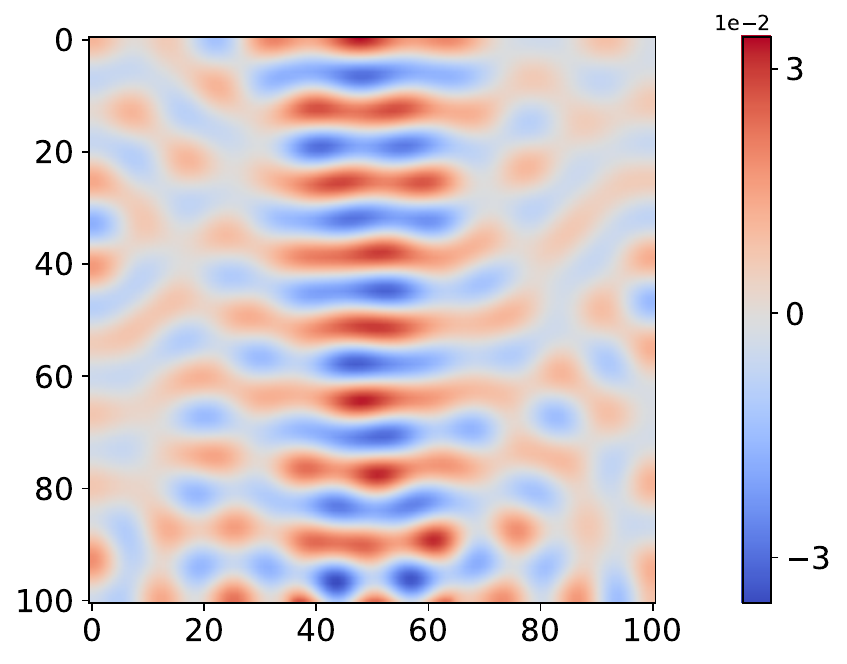}}
\subfigure[True test]{\includegraphics[width=28mm]{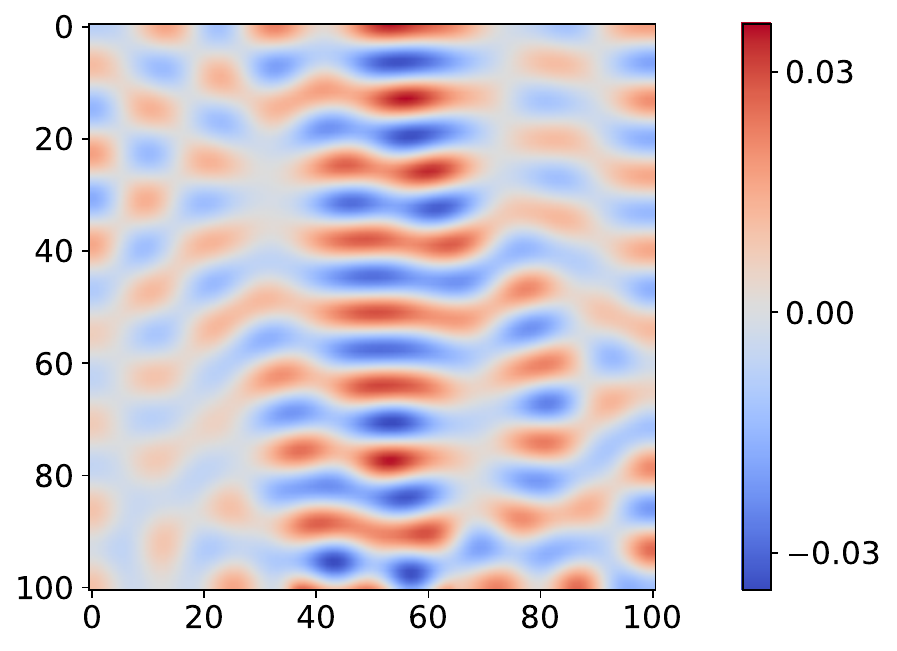}}
\subfigure[Predicted test]{\includegraphics[width=28mm]{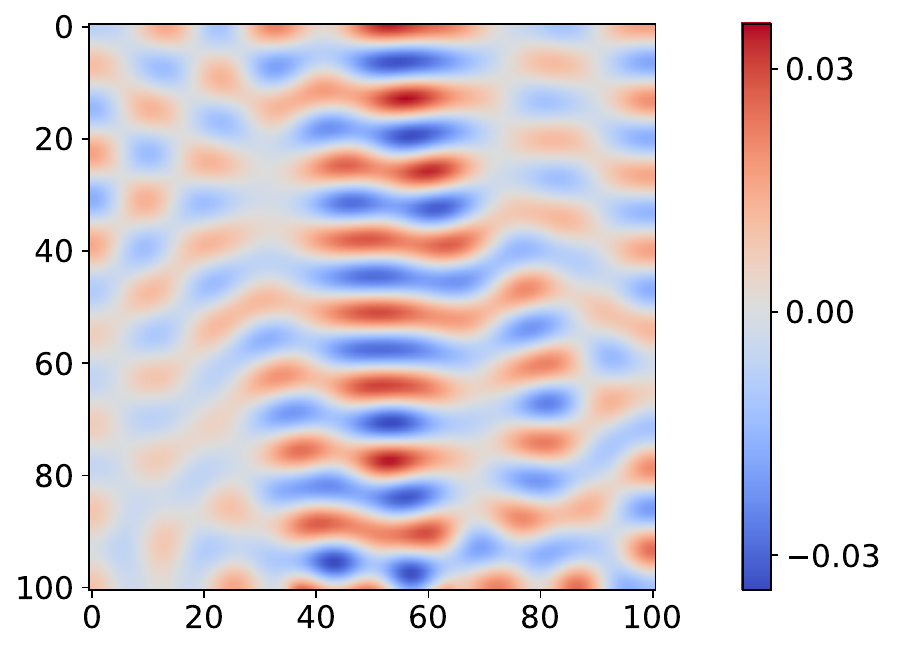}}
\subfigure[Pointwise error]{\includegraphics[width=28mm]{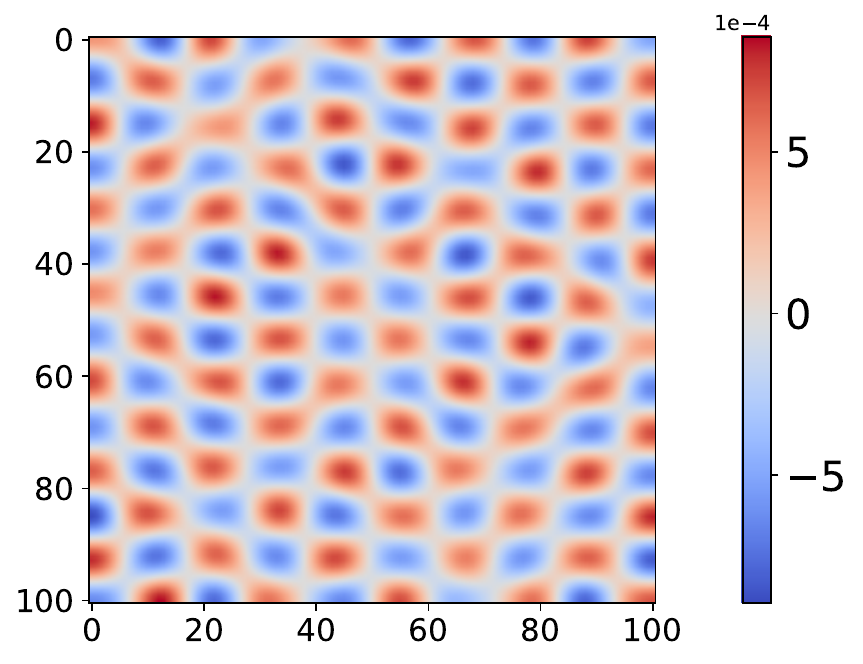}}
\caption{Helmholtz's equation: an example of training input, training output, true test, predicted test using Cauchy random features, and pointwise error. We randomly generate $N=5000$ i.i.d samples from tensor-product Cauchy distribution with scaling parameter $\gamma = 10^{-4}$. }
\label{Fig: Helmholtz}
\end{figure}

\subsection{Structural Mechanics}

The system that governs the displacement vector $w$ in an elastic solid undergoing infinitesimal deformations is defined as 
\begin{equation}
\begin{aligned}
    & \nabla\cdot \sigma = 0 \quad && \mbox{ in } D \\
    & w=\Bar{w} \quad && \mbox{ on } \Gamma_w \\
    & \sigma\cdot n = u \quad &&\mbox{ on } \Gamma_u
\end{aligned}
\end{equation}
where $\sigma$ is the (Cauchy) stress tensor and $n$ is the outward unit normal vector. The computational domain is denoted by $D=(0,1)^2$, and its boundary $\partial D$ is split in $[0,1]\times 1 = \Gamma_u$ and its complement $\Gamma_w$. The prescribed displacement $\Bar{w}$  and the surface traction $u$ are imposed on the domain boundaries $\Gamma_w$ and $\Gamma_u$, respectively.
We aim to learn the operator that maps the one-dimensional load $u$ on $\Gamma_u$ to the two-dimensional von Mises stress field $v$ on $D_\Vc=(0,1)^2$. 
The load $u$ is drawn from Gaussian Process $GP(100, 400^2(-\Delta+\tau^2\Id)^{-d})$ with $\Delta$ being the Laplacian subject to homogeneous Neumann boundary conditions on the space of spatial-mean zero functions. The inverse length scale of the random field is taken to be $\tau=3$ and $d = 1$ determines its regularity (upto 1/2 a fractional derivative for samples from this measure). 
The load $u$ is interpolated on a 41 grid and extruded in the $y$ direction and the stress field is interpolated on a 21 $\times$ 21 grid via radial basis function interpolation.
We use 1000 functions for training and 1000 functions for testing. 
Figure \ref{Fig:Structure} shows an example of training data, true test function and its approximation using Cauchy random features along with the corresponding pointwise error. 

\begin{figure}[!htbp]
\centering     
\subfigure[Training input]{\includegraphics[width=28mm]{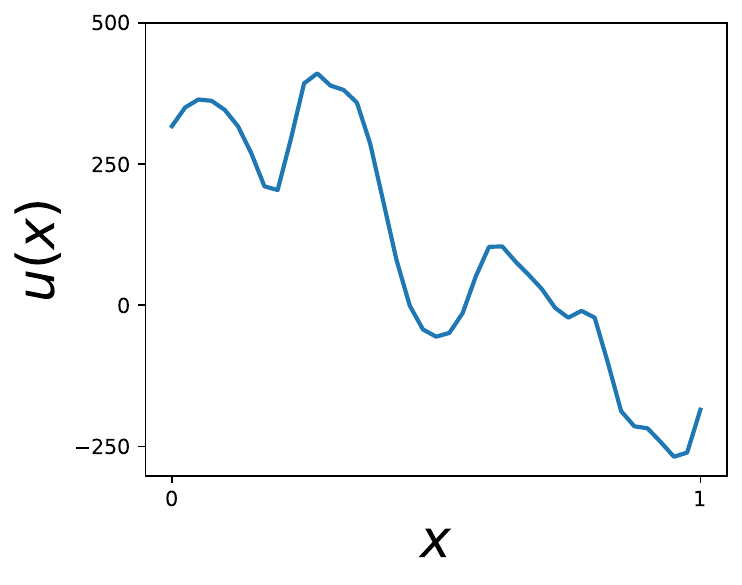}}
\subfigure[Training output]{\includegraphics[width=28mm]{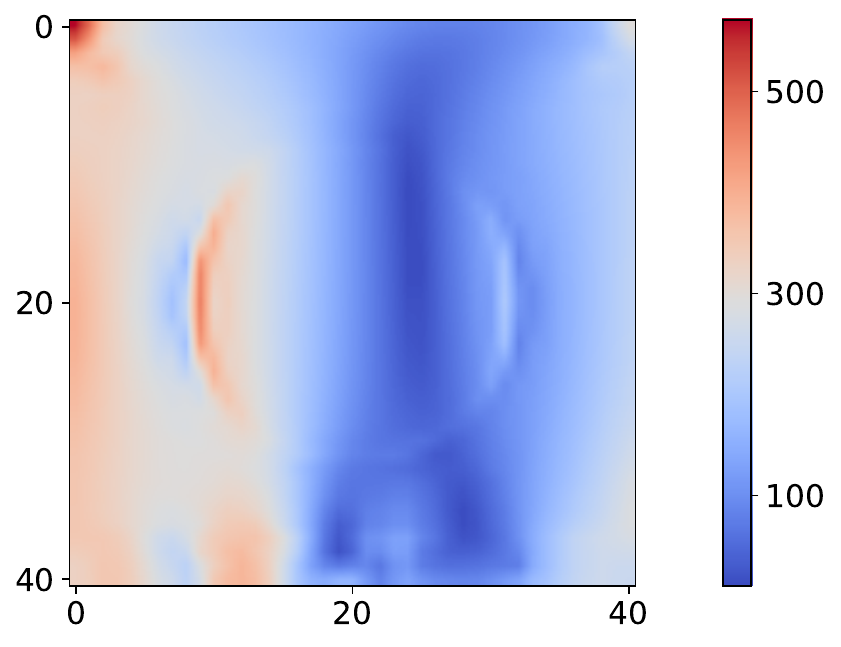}}
\subfigure[True test]{\includegraphics[width=28mm]{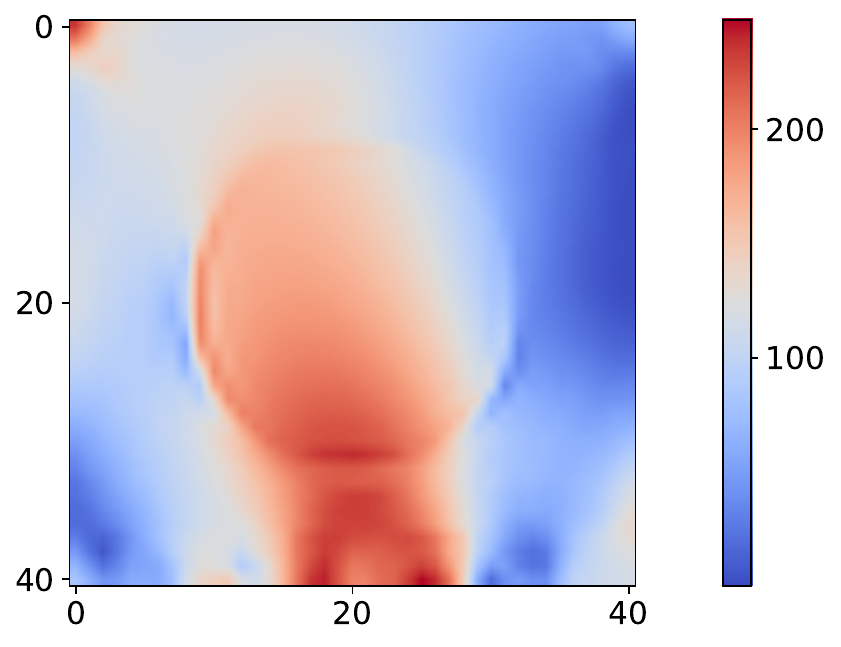}}
\subfigure[Predicted test]{\includegraphics[width=28mm]{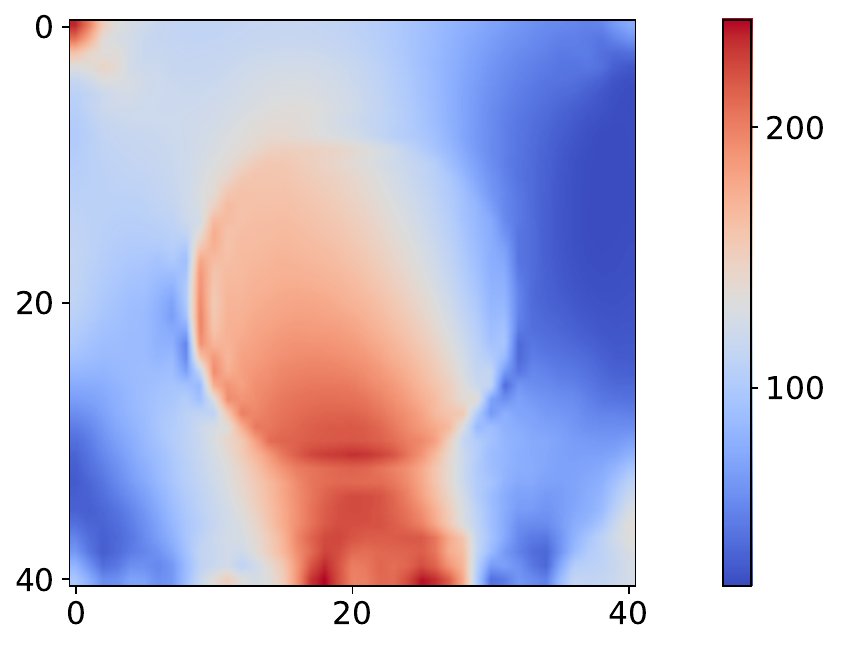}}
\subfigure[Pointwise error]{\includegraphics[width=28mm]{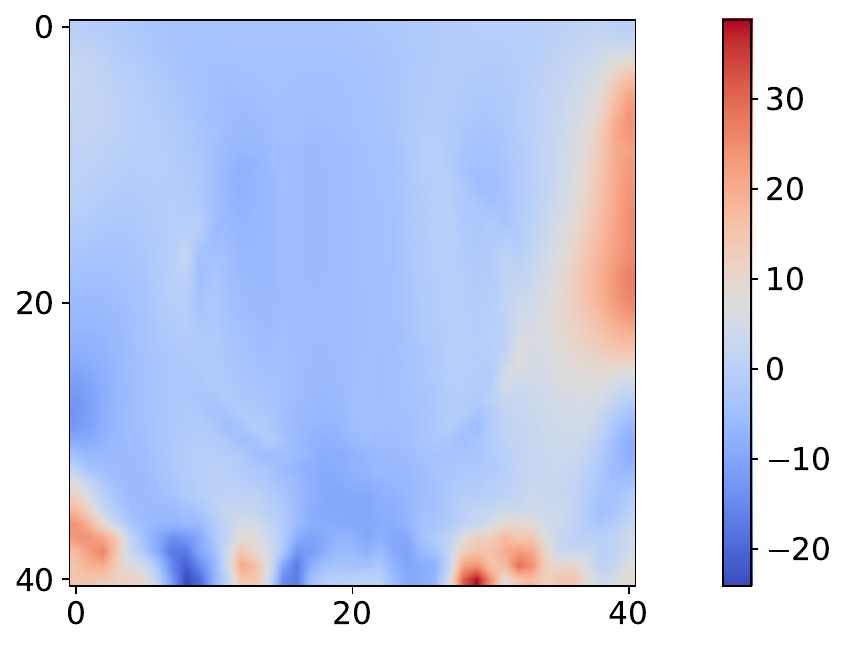}}
\caption{Structural mechanics: an example of training data, true test function, predicted test using Cauchy random features, and pointwise error.}
\label{Fig:Structure}
\end{figure}

\subsection{Navier-Stokes Equation}

We consider the vorticity-stream $(\omega,\psi)$ formulation of the incompressible Navier-Stokes equation:
\begin{equation}
\begin{aligned}
    & \frac{\partial w}{\partial t} + (c\cdot \nabla)w - v\Delta w = u \\
    & w=-\Delta\psi \qquad \int_{D}\psi=0 \\
    & c = (\frac{\partial\psi}{\partial x_2}, -\frac{\partial\psi}{\partial x_1}),
\end{aligned}
\end{equation}
where computational domain is $D=[0,2\pi]^2$ and periodic boundary conditions are considered.
We are interested in learning the operator $G:u\mapsto w(\cdot,T)$, where $u$ is the forcing term $u$ and $w(\cdot,T)$ is the vorticity field at a given time $t=T$. The forcing term $u$ is sampled from Gaussian Process $GP(0, (-\Delta+\tau^2\Id)^{-d})$. As previous examples, $\Delta$ denotes the Laplacian on $D$ subject to periodic boundary conditions on the space of spatial-mean zero functions, $\tau=3$ denotes the inverse length scale of the random field and $d = 4$ determines its regularity; the choice of $d$ then leads to up to 3 fractional derivatives for samples from this measure. We fix the initial condition $w(\cdot,0)$ which is generated from the same distribution.
Given the constant viscosity $\nu=0.025$, the equation is solved on a $16\times16$ grid with a pseudo-spectral method and Crank-Nicholson time integration. The size of the training dataset is 2000 and the test dataset is of size 2000. 
Figure \ref{Fig:NS} shows an example of training data, true test function, its approximation using Cauchy random features, and the corresponding pointwise error. 

\begin{figure}[!htbp]
\centering     
\subfigure[Training input]{\includegraphics[width=28mm]{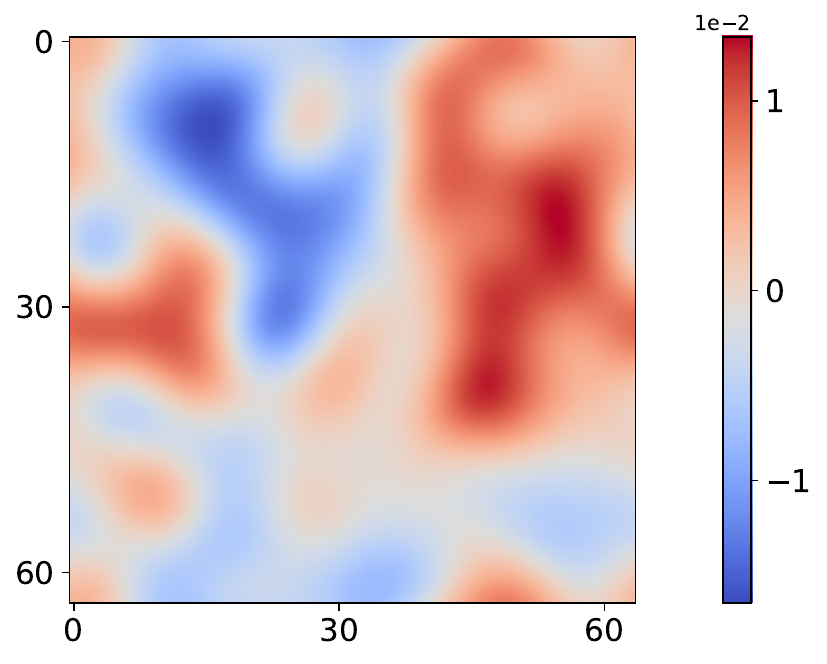}}
\subfigure[Training output]{\includegraphics[width=28mm]{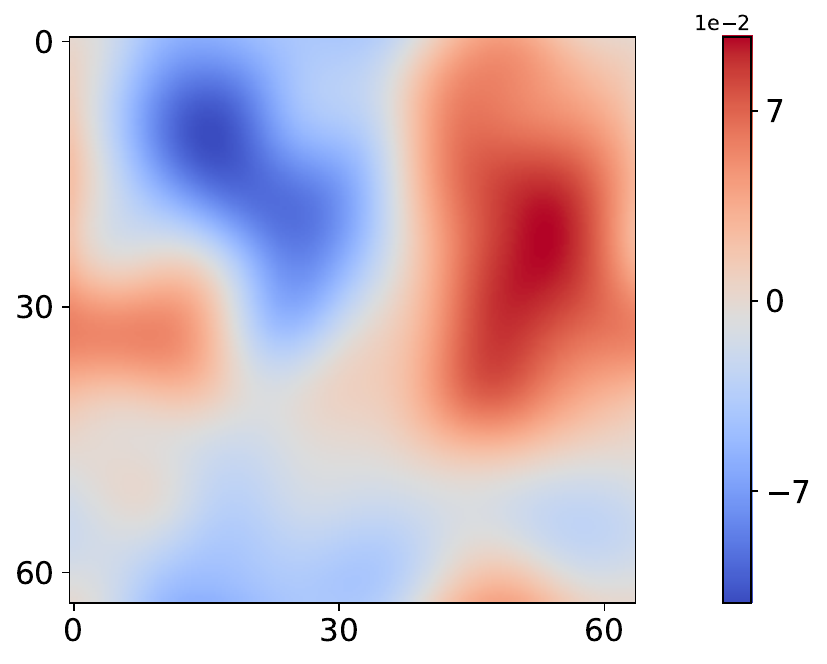}}
\subfigure[True test]{\includegraphics[width=28mm]{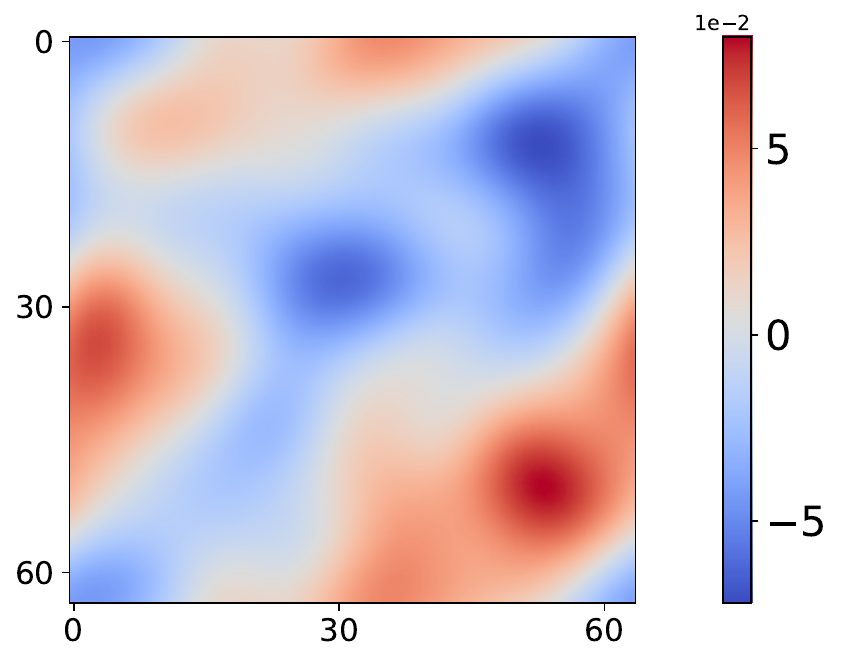}}
\subfigure[Predicted test]{\includegraphics[width=28mm]{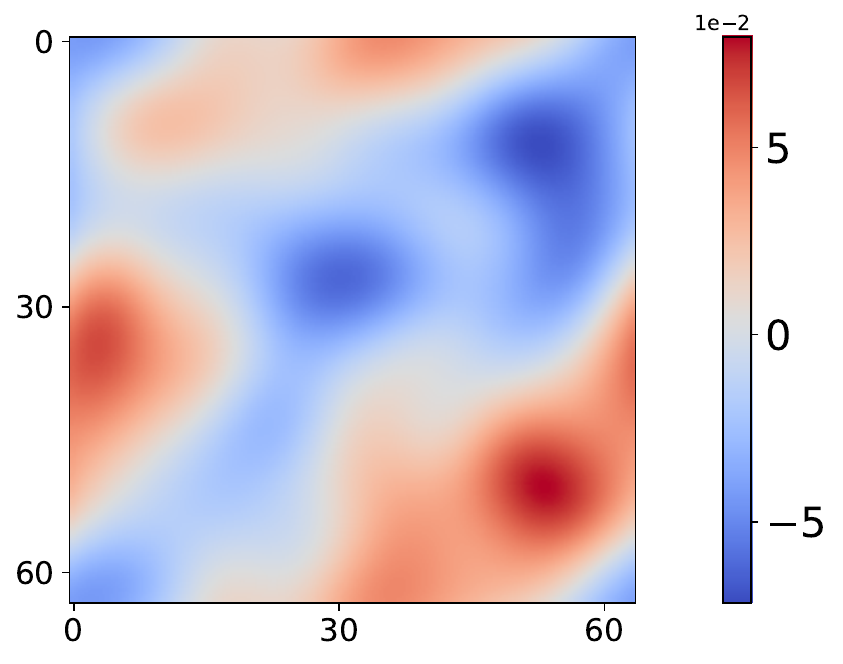}}
\subfigure[Pointwise error]{\includegraphics[width=28mm]{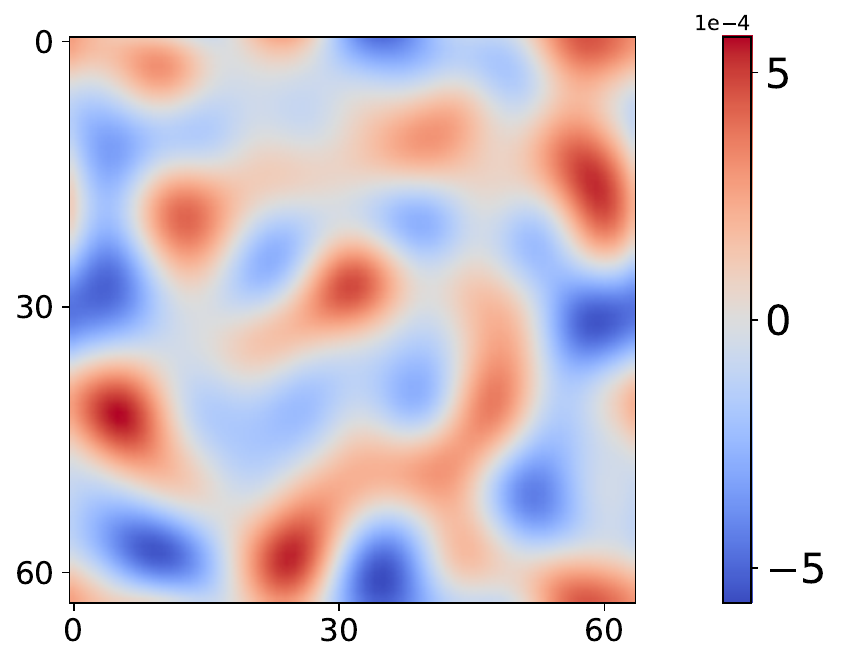}}
\caption{Example of training data, true test and predicted test using Cauchy random feature method, and pointwise error for Navier-Stokes equation.}
\label{Fig:NS}
\end{figure}

\subsection{Results}
In Table \ref{Table:RF_kernel}, we report the relative test errors of random feature methods and of kernel methods. 
Even though random feature methods can be viewed as an approximation of a kernel method, we observe that the random feature method achieves similar relative test errors or even improves the accuracy over the kernel method in some tests. 
Both the random feature method and kernel method are easy to implement in practice. 
However, the implementation of kernel methods requires high-performance numerical computing and machine learning python library such as JAX, see \cite{BATLLE2024kernel}. 
The implementation of our proposed random feature method can be done by using scikit-learn library and GPU computing is not required. 
Therefore, our proposed algorithm achieves similar accuracy even if the computational resources are limited.
As concluded in \cite{BATLLE2024kernel}, the kernel method either matches the performance of neural network methods or outperforms them in several benchmarks. Therefore, the random feature method can match the performance or outperform neural network methods as well.

\begin{table}[!htbp]
\centering
\begin{tabular}{|c|c|c|c|c|}
\hline
Example & RBF Kernel & Mat{\' e}rn Kernel  &  RF(Cauchy) & RF(Gaussian) \\\hline
Burgers' &  3.76 & {\bf 2.03} & 3.82 & 2.70 \\\hline
Darcy & 4.93 & 4.47 & {\bf 3.08} & 3.74 \\\hline
Helmholtz & 5.05 & 3.76 & 6.66 & {\bf 3.63} \\\hline
Structural Mechanics & 10.31 & 7.73 & {\bf 7.67} & 8.71 \\\hline
Navier-Stokes & 1.09 & {\bf 0.91} & 2.54 & 1.06 \\\hline
\end{tabular}
\caption{Summary of relative test errors of random features methods and kernel methods.}
\label{Table:RF_kernel}
\end{table}
We further compare the performances of the kernel method and the random feature method provided that each model is trained over a similar time period. 
Since the training time of the kernel method depends on the number of training samples, we further subsample a small dataset to train the kernel methods\.
In Table \ref{Table:RF_kernel2}, we report the relative test errors and training times of kernel methods and of random feature methods. We observe that our proposed random feature method outperforms kernel method cross all benchmarks provided that the same amount of training time is allocated to each model.

\begin{table}[!htbp]
\footnotesize
\centering
\begin{tabular}{|c|c|c|c|c|c|}
\hline
&  & RBF Kernel & Mat{\' e}rn Kernel  &  RF(Cauchy) & RF(Gaussian) \\ \hline
\multirow{2}{*}{Burgers'} & Relative error (\%) & 6.87 & 5.33 & {\bf 3.85} & 4.10 \\\cline{2-6} 
& Training time (seconds) & 3.9 & 4.8 & 3.1 & 3.2 \\\hline
\multirow{2}{1.2cm}{Navier- \\Stokes} & Relative error (\%) & 3.01 & 3.05 & 2.52 & {\bf 1.06} \\\cline{2-6} 
& Training time (seconds) & 2.4 & 3.1 & 3.0 & 1.9 \\\hline
\end{tabular}
\caption{Summary of relative test errors of random feature methods and kernel methods given the same amount of training times. For both Burgers' equation problem and Navier-Stokes equation problem, 500 samples are used for training and another 500 samples are used for testing.}
\label{Table:RF_kernel2}
\end{table}

We also compare our proposed random feature method with DeepONet and FNO directly in terms of the accuracy, see the summary of test relative errors in Table \ref{Table:Test_error}. We report the training times of random feature methods as well. The relative errors of DeepONet and FNO are cited from \cite{LU2022operator, dehoop2022costaccuracy, BATLLE2024kernel}. 
We observe that the random feature methods can be trained fast even if the problem is complicated. 
For example, there are 20000 training samples in the Navier-Stokes problem and each training function is interpolated on a $41\times 41$ grid. 
The random feature method matches the performance of DeepONet and FNO in the structural mechanics example, and outperforms DeepONet in the Helmholtz example, and outperforms both DeepONet and FNO in the Navier-Stokes example.

\begin{table}[!htbp]
\footnotesize
\centering
\begin{tabular}{|c|c|c|c|c|c|}
\hline
&  & DeepONet & FNO  &  RF(Cauchy) & RF(Gaussian) \\ \hline
\multirow{2}{*}{Helmholtz} & Relative error (\%) & 5.88 & 1.86 & 2.54 & 2.92 \\\cline{2-6} 
& Training time (seconds) & - & - & 21.8 & 12.5 \\\hline
\multirow{2}{1.5cm}{Structural \\Mechanics} & Relative error (\%) & 5.20 & 4.76 & 6.08 & 6.25 \\\cline{2-6} 
& Training time (seconds) & - & - & 4.2 & 5.8 \\\hline
\multirow{2}{1.5cm}{Navier- \\Stokes} & Relative error (\%) & 3.63 & 0.26 & 0.73 & 0.11 \\\cline{2-6} 
& Training time (seconds) & - & - & 24.9 & 23.9 \\\hline
\end{tabular}
\caption{Summary of relative test errors of random feature methods and neural operator benchmarks DeepONet and FNO. We also report the training times of random feature methods.}
\label{Table:Test_error}
\end{table}

The choice of discretizations is important in operator learning. 
A true operator learning method should demonstrate the invariance of the expected relative test error
to the mesh resolution used for training and testing.
We empirically verify that our proposed method is stable to discretization changes, and hence is discretization-invariant, see also \cite{zhang2025discretization} for a discussion on discretization-invariant operator networks.
We present the relative test errors at different resolutions for Darcy problem and Helmholtz problem in Tables \ref{tab:darcy_reso} and \ref{tab:Hel_reso}, respectively. 
For each problem, we use 1000 functions for training and 1000 functions for testing.

\begin{table}[!htbp]
\centering
\begin{tabular}{|c|c|c|c|c|}
\hline
Resolution & $13\times13$ & $26\times26$ & $51\times51$ & $101\times101$  \\\hline
Relative Test Error (\%) & 8.18 & 8.35 & 8.30 & 8.53 \\\hline
\end{tabular}
\caption{Darcy problem: relative test error at different resolutions.}
\label{tab:darcy_reso}
\end{table}

\begin{table}[!htbp]
\centering
\begin{tabular}{|c|c|c|c|c|}
\hline
Resolution & $29\times29$ & $31\times31$ & $36\times36$ & $43\times43$  \\\hline
Relative Test Error (\%) & 4.06 & 3.95 & 3.88 & 3.98\\\hline
\end{tabular}
\caption{Helmholtz problem: relative test error at different resolutions.}
\label{tab:Hel_reso}
\end{table}

In Figure \ref{Fig:rate}, we empirically verify the error bound we obtained in Theorem \ref{OL_error} using some benchmarks. As the number of random features $N$ increases, we observe that the test errors are all convergent as a rate of $1/\sqrt{N}$. We also observe that the growth rate of training time is between $\sqrt{N}$ and $N$. 

\begin{figure}[!htbp]
\centering     
\subfigure[Decay rate of generalization error]{\includegraphics[width=73mm]{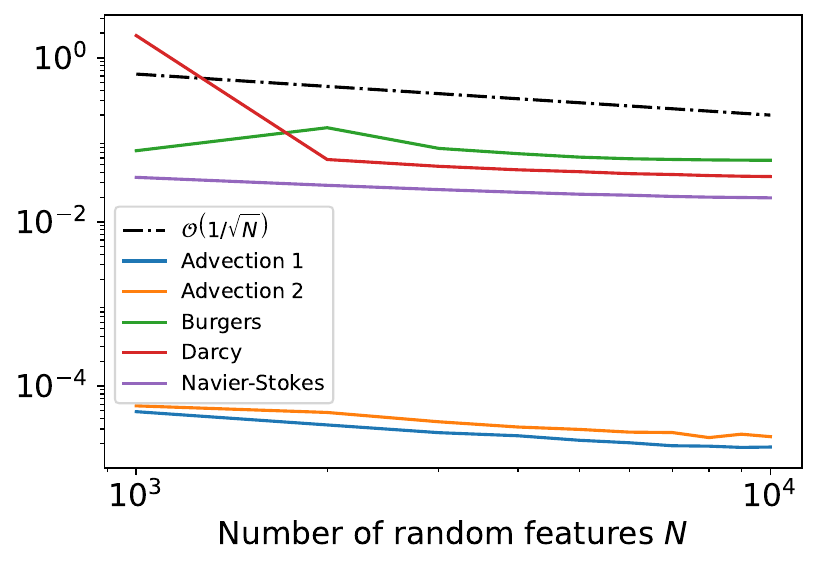}}
\subfigure[Growth rate of training time]{\includegraphics[width=73mm]{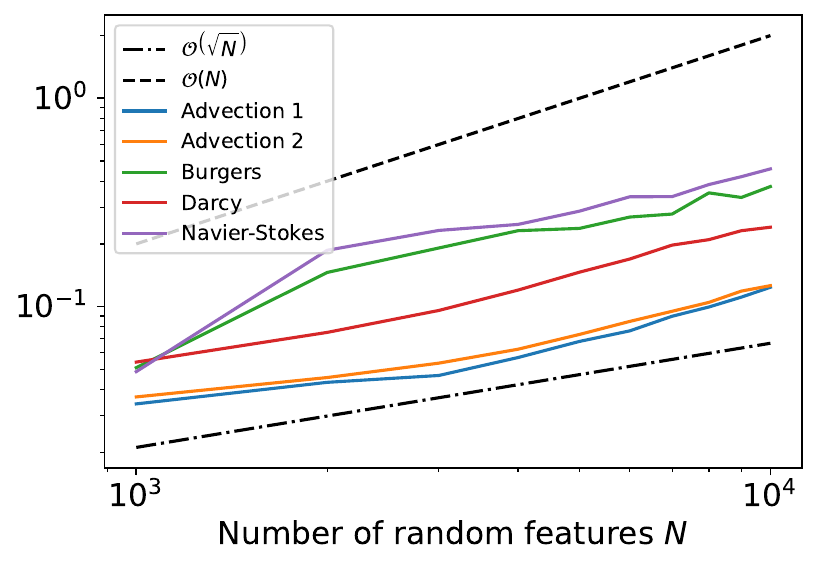}}
\caption{Empirical verifications of (a) decay rate of generalization error and (b) growth rate of training time. We repeat each experiment 10 trials to average.}
\label{Fig:rate}
\end{figure}

\section{Conclusion}
We propose a random feature method for learning nonlinear Lipschitz operators related to PDE.
We provide a detailed error analysis for our proposed method with Cauchy random features. The theory suggests that the generalization error decays as a rate of $1/\sqrt{N}$, where $N$ is the number of random features. The results hold for both Gaussian and Cauchy distributions. 
Numerical experiments show that the random feature method not only enjoys the easy and fast implementation, but also matches or outperforms kernel method and neural network methods in benchmark 2D PDE examples. 
In addition, the theoretical and computational analysis shows the benefits of randomization.  The use of Cauchy features supports the application to PDE. These methods may also be applicable to the recent multi-operator learning approaches, for example, the transformer based models \cite{liu2024prose, sun2024towards, liu2024prosefd,liu2025bcat, cao2024vicon}
or DeepONet based approaches \cite{zhang2024deeponet}.
Future work could explore other heavy tail random features and noisy data. 

\section*{Acknowledgments}
HS was partially supported by NSF DMS 2331033. DN was partially supported by NSF DMS 2408912.

\bibliography{refs}
\end{document}